%% file: 00_main.tex
\documentclass{article} 
\usepackage{iclr2027_conference,times}

\input{math_commands.tex}

\usepackage{hyperref}
\usepackage[utf8]{inputenc} 
\usepackage[T1]{fontenc}    
\usepackage{url}            
\usepackage{amssymb}
\usepackage{booktabs}       
\usepackage{amsfonts}       
\usepackage{nicefrac}       
\usepackage{microtype}      
\usepackage{xcolor}         
\usepackage{tcolorbox}
\usepackage{graphicx}
\usepackage{subfigure}
\usepackage{subcaption}
\usepackage{enumerate}
\usepackage{amsmath}
\usepackage{amssymb}
\usepackage{mathtools}
\usepackage{amsthm}
\usepackage[table]{xcolor}
\usepackage{algorithm}
\usepackage{algorithmic}
\usepackage{titlesec}
\usepackage{wrapfig}
\usepackage{caption}
\usepackage[normalem]{ulem}
\usepackage{multirow}
\usepackage{array}
\usepackage{multicol}
\usepackage{tikz}
\usepackage{makecell}
\theoremstyle{plain}
\newtheorem{theorem}{Theorem}[section]
\newtheorem{proposition}[theorem]{Proposition}

\newtheorem{corollary}[theorem]{Corollary}
\theoremstyle{definition}
\newtheorem{definition}[theorem]{Definition}

\theoremstyle{remark}
\newtheorem{remark}[theorem]{Remark}

\def\method{\textbf{Sculpt}}
\def\bench{\textbf{UrbanWindFlow}}
\definecolor{myred}{RGB}{243,125,119}
\definecolor{deepdeepred}{RGB}{107,0,12}
\definecolor{mydeepred}{RGB}{178,23,41}
\definecolor{mygreen}{RGB}{112,195,121}
\definecolor{myblue}{RGB}{0,148,254}
\definecolor{mypurple}{RGB}{128,128,254}
\definecolor{mydeepblue}{RGB}{117,112,179}

\definecolor{thered}{RGB}{247,105,108}   
\definecolor{thegreen}{RGB}{99,190,123}  
\definecolor{myred}{HTML}{D3393C}
\definecolor{mygreen}{HTML}{3A9D5D}

\newcommand{\colorcell}[1]{%
  \pgfmathparse{#1}%
  \ifdim\pgfmathresult pt<0pt
    \pgfmathsetmacro{\val}{min(100, abs(#1)/30 * 100)} 
    \edef\colorval{\noexpand\cellcolor{thered!\val!white}}%
  \else
    \pgfmathsetmacro{\val}{min(100, #1/631.4 * 100)} 
    \edef\colorval{\noexpand\cellcolor{thegreen!\val!white}}%
  \fi
  \colorval
}

\newcommand{\calimp}[1]{%
    \ifdim#1 pt > 0pt %
        \textcolor{mygreen}{\tiny ↑\the\numexpr #1\relax\%}%
    \else
        \textcolor{myred}{\tiny ↓\the\numexpr -#1\relax\%}%
    \fi
}

\title{Shaping the Wind: Nested Potentials for Kinematically Admissible Urban Wind Prediction}

\author{
Yidi Wang$^{1}$ \quad
Yunhe Zhang$^{1}$ \quad
Jiawei Gu$^{2}$ \quad
Ziyue Qiao$^{2}$ \quad
Pengyang Wang$^{1}$\thanks{Corresponding Author.}\\
$^{1}$The State Key Laboratory of Internet of Things for Smart City, Department of Computer and\\
Information Science, University of Macau\\
$^{2}$School of Computing and Information Technology, Great Bay University\\
\texttt{wang.yidi@connect.um.edu.mo \quad pywang@um.edu.mo}
}

\iclrfinalcopy 
\begin{document}

\maketitle

\input{latex/chapters/0_abstract}

\input{latex/chapters/1_intro}

\input{latex/chapters/2_pre}
\input{latex/chapters/3_method}

\input{latex/chapters/4_experiment}
\input{latex/chapters/5_conclusion}

\input{latex/chapters/6_others}

\input{latex/chapters/8_ack}

\bibliography{iclr2027_conference}
\bibliographystyle{iclr2027_conference}

\input{latex/chapters/7_appendix}

\end{document}

%% file: math_commands.tex
\usepackage{amsmath,amsfonts,bm}

\def\eqref#1{equation~\ref{#1}}

\def\1{\bm{1}}

\DeclareMathAlphabet{\mathsfit}{\encodingdefault}{\sfdefault}{m}{sl}
\SetMathAlphabet{\mathsfit}{bold}{\encodingdefault}{\sfdefault}{bx}{n}



%% file: latex/chapters/0_abstract.tex
\begin{abstract}
Predicting transient urban winds is fundamental to understanding urban microclimates and designing climate-resilient cities. Building-resolving large-eddy simulation produces detailed incompressible urban wind fields at substantial computational cost for each layout. Neural surrogates offer a faster alternative by learning to predict the evolution of velocity fields. However, minimizing velocity prediction error does not guarantee local mass conservation and wall impermeability, which together define kinematic admissibility. This limitation stems from an unconstrained output representation: geometry conditioning guides predictions but does not restrict them to admissible velocity fields. Correcting boundary violations in these outputs changes the flux balance in adjacent fluid cells and may consequently compromise local mass conservation. To address the challenge, we propose \method{}, a nested potential framework that builds the coupled, geometry-dependent constraints directly into its parameterization.
This nested parameterization generates divergence-free velocity updates through the discrete curl of a volume vector potential on the native three-dimensional staggered grid. A shared scalar potential constrains the vector potential's boundary values so that the same operator also enforces impermeability, without a per-step pressure projection. Because backpropagation through this curl attenuates large-scale gradient signals, we parameterize the volume potential at multiple resolutions to better capture large-scale flow structures. We introduce \bench{}, an LES dataset spanning urban morphologies and inflow conditions, to evaluate accuracy and kinematic admissibility together. Across all three evaluation settings, \method{} consistently outperforms both the reduced-order and the neural baselines, reducing mean squared error by $11.7$--$39.3\%$ while maintaining constraint residuals near numerical precision.
\end{abstract}

%% file: latex/chapters/1_intro.tex
\section{Introduction}

Predicting transient urban winds is essential for understanding urban microclimates and informing climate-resilient city design. Buildings redirect the incompressible flow into wakes, recirculation zones, and street-level jets whose evolution depends on both layout and atmospheric conditions~\citep{britter2003flow,blocken2012cfd}. Building-resolving large-eddy simulation (LES) captures these three-dimensional dynamics, but repeated solves across layouts and inflow conditions are computationally demanding~\citep{okaze2026aij}. Neural surrogates offer a faster route by learning to advance velocity fields from their recent history and building geometry~\citep{brunton2020machine,gupta2022towards,lippe2023pde}. Their practical value depends on predicting fields that remain compatible with the surrounding geometry.

\begin{wrapfigure}{R}{0.55\textwidth}
  \vskip -0.1in
  \centering
  \includegraphics[width=0.54\textwidth]{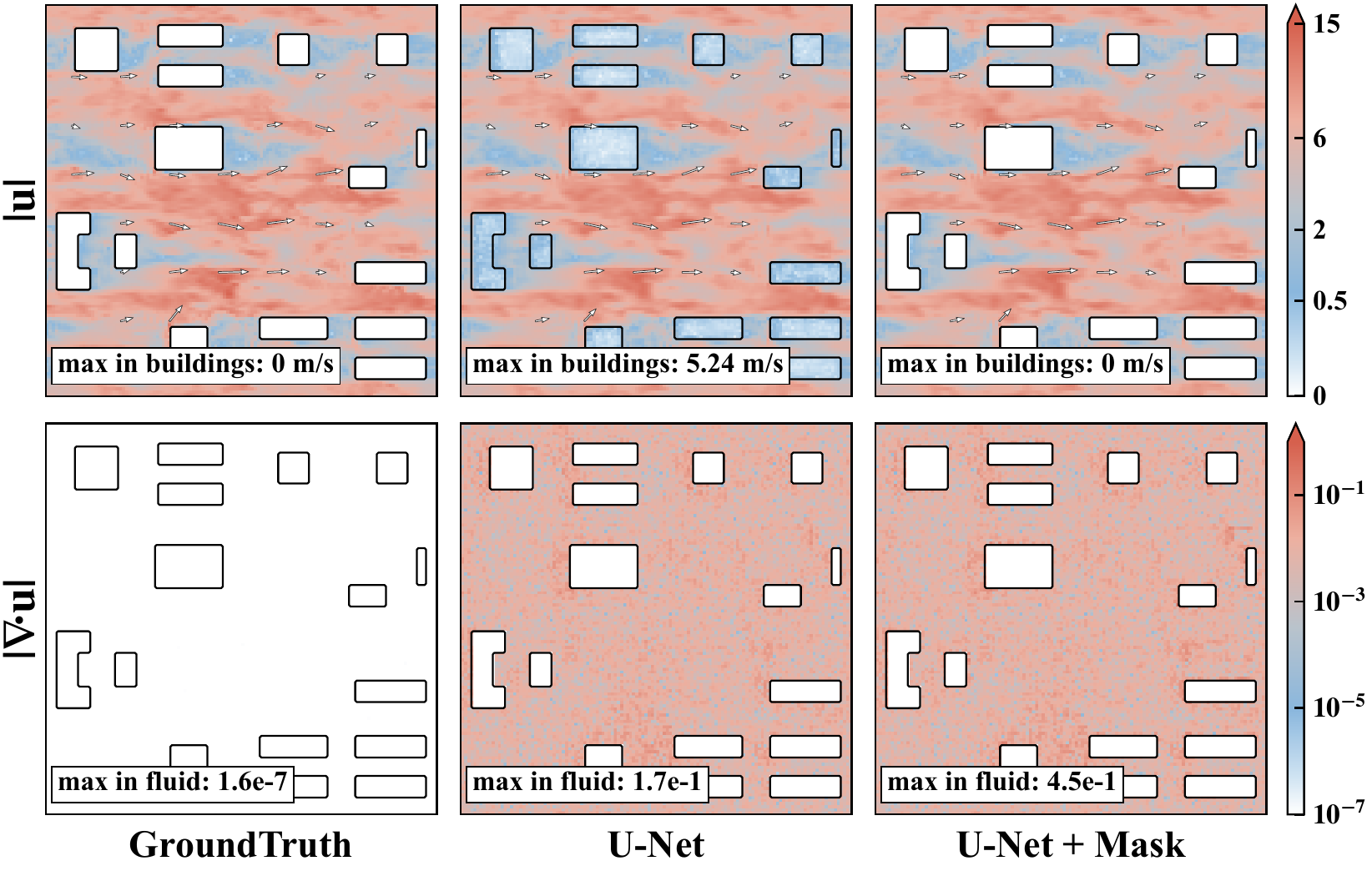}
  \caption{Wind speed (top) and divergence magnitude (bottom) for the ground truth and U-Net predictions before and after enforcing wall impermeability (with direct building mask). Boundary correction eliminates flow through buildings but amplifies divergence near boundaries, compromising local mass conservation.
  }
  \label{fig:motivate}
  \vskip -0.1in
\end{wrapfigure}

Geometry conditioning alone does not ensure this compatibility. Local mass conservation requires balanced fluxes in each incompressible fluid cell, while wall impermeability requires zero normal flow through solid surfaces.
Here, \emph{kinematic admissibility}~\citep{drucker1952extended,batchelor2000introduction} refers to the joint satisfaction of these two conditions. They constrain shared velocity degrees of freedom: changing a wall-face velocity also changes the flux balance in its adjacent cell. Fig.~\ref{fig:motivate} illustrates this coupling: a U-Net prediction closely resembles the reference field but permits flow through buildings; masking the solid region suppresses penetration while amplifying divergence near building edges. The challenge is to ensure kinematic admissibility across different building geometries.

Existing approaches offer different ways to incorporate these constraints. Soft penalties encourage their satisfaction during training~\citep{raissi2019physics,karniadakis2021physics}, while a geometry-aware pressure projection enforces both through an additional solve at each prediction step. Structured output representations impose conservation directly, including through vector potentials and divergence-free neural fields~\citep{beucler2021enforcing,wandel2021teaching,richter2022neural}. Geometry-dependent compatible spaces have also been developed for steady PDE problems~\citep{shaffer2026structure}. For an explicit urban-flow surrogate, the central construction problem is to encode solid boundaries in the potential while retaining a divergence-free update on the native three-dimensional grid. A curl guarantees zero divergence, but impermeability depends on the potential's boundary values.

To address this representational challenge, we propose \method{}, a nested potential framework that enforces kinematic admissibility through a geometry-dependent parameterization of velocity updates. For a layout $g$, discrete conservation and wall conditions define a linear operator $\mathcal{C}(g)$ and an admissible space $\mathcal{V}(g)=\ker\mathcal{C}(g)$. \method{} maps unconstrained potential variables into this space through nested volume and boundary potentials on the native three-dimensional staggered grid. The discrete curl of the vector potential yields divergence-free increments through cancellation of edge contributions within each cell. To enforce impermeability through the same operator, the vector potential's boundary values are assembled from differences of a shared scalar potential, together with compatible global circulation modes. This makes circulation around every boundary face vanish, ensuring zero normal flow. A volumetric module and a surface graph network predict these potential variables, which are assembled before the curl readout. Starting from an admissible state, the resulting updates preserve both constraints throughout a rollout on the fixed layout without a per-step pressure projection. This representation also shapes learning: backpropagation of the velocity loss through the adjoint of the curl reweights gradients across spatial scales and attenuates large-scale signals. We therefore parameterize the volume potential at multiple resolutions, giving coarse flow structures dedicated prediction branches while retaining the same admissible readout.

To evaluate prediction accuracy and kinematic admissibility together, we introduce \bench{}, a dataset of 85 building-resolving LES sequences spanning urban morphologies and inflow conditions. Its three evaluation settings separate unseen layout instances, unseen morphology classes, and held-out inflow combinations. Across these settings, \method{} reduces 16-step mean squared error by 11.7–39.3\% relative to nine baselines while maintaining constraint residuals near numerical precision. The nested potentials yield lower error than penalty and projection variants, while multiresolution potentials improve accuracy under three shifts.

Our contributions are threefold:
\begin{enumerate}[0]
\item[$\bullet$] \textbf{A framework for learning kinematic admissibility.} We introduce a nested volume-boundary potential construction that jointly enforces discrete mass conservation and wall impermeability during autoregressive urban wind-flow prediction.
\item[$\bullet$] \textbf{Learning effects of physical structure.} We characterize the scale dependence of curl-induced gradients and develop multiresolution potentials to support the learning of large-scale dynamics.
\item[$\bullet$] \textbf{Joint evaluation of accuracy, admissibility, and generalization.} We provide a public LES dataset and evaluation protocol that assess prediction accuracy, kinematic admissibility, and transfer across urban morphologies and inflow conditions.
\end{enumerate}

%% file: latex/chapters/2_pre.tex
\section{Preliminaries}

In building-resolving urban wind modeling, we consider the wind in a city neighborhood as an incompressible fluid, discretized on a uniform staggered grid that is periodic in the two horizontal directions and bounded below by the ground and above by a lid. The building layout is a per-cell occupancy mask that marks each cell as fluid or solid. We denote that mask by $g$ and $\Omega(g)$ for the set of cells it marks as fluid. The instantaneous velocity $\mathbf{u}$ is stored on fluid cell faces of the Arakawa C-grid~\citep{harlow1965numerical}. The layout therefore determines the fluid cells and the boundary faces on which the velocity field must satisfy the kinematic conditions defined below.

\subsection{Kinematically Admissible Urban Wind Field}
\begin{definition}[Kinematic Admissibility] \label{def:ka}
A velocity field $\mathbf{u}$ is kinematically admissible~\citep{drucker1952extended,batchelor2000introduction} for the layout $g$ when it jointly satisfies the two conditions over the whole geometry. The first is local mass conservation in every fluid cell. The second is wall impermeability, meaning that no flow can cross an impermeable surface:
\begin{equation} \label{eq:pre-admissible}
    (\nabla_h\!\cdot\!\mathbf{u})_c = 0 \quad \forall\, c\in\Omega(g),
    \qquad\qquad
    u_f = 0 \quad \forall\, f\in\mathcal{F}_{\partial}(g),
\end{equation}
where $c$ ranges over the fluid cells $\Omega(g)$, $(\nabla_h\!\cdot\!\mathbf{u})_c$ is the finite-volume divergence of that cell, $f$ ranges over $\mathcal{F}_{\partial}(g)$, the set of impermeable faces, namely the building surfaces, the ground and the lid, and $u_f$ is the velocity component normal to $f$. Only that normal component is constrained.
\end{definition}

Because both conditions in Eq.~(\ref{eq:pre-admissible}) are linear in $\mathbf{u}$, they can be written jointly as
$\mathcal{C}(g)\mathbf{u}=\mathbf{0}$,
where $\mathcal{C}(g)$ is a sparse linear operator read off from the mask. Their solutions constitute the admissible space
$\mathcal{V}(g)=\ker\mathcal{C}(g)\subset\mathbb{R}^{N}$,
where $\ker\mathcal{C}(g)$ denotes the kernel of $\mathcal{C}(g)$ and $N$ is the total number of faces of the fluid cells. This formulation highlights three properties of the admissible fields. \textbf{(1)} $\mathcal{V}(g)$ is fixed by the geometry: a different layout corresponds to a different subspace rather than a different value of a conditioning variable. \textbf{(2)} For a given layout $g$, every true wind field $\mathbf{u}$ should belong to $\mathcal{V}(g)$ because it satisfies both kinematic conditions. Hence, for any prediction $\hat{\mathbf{u}}$ of that field, any component outside $\mathcal{V}(g)$ cannot match the true field and can only add to the prediction error~\citep{beucler2021enforcing,zanetta2023physics}. Mathematically, the orthogonal decomposition gives $\|\hat{\mathbf{u}}-\mathbf{u}\|^{2}=\|\Pi_g\hat{\mathbf{u}}-\mathbf{u}\|^{2}+\|(\mathbf{I}-\Pi_g)\hat{\mathbf{u}}\|^{2},$ where $\Pi_g$ denotes the orthogonal projector onto $\mathcal{V}(g)$. \textbf{(3)} The two kinematic conditions are coupled through shared velocity degrees of freedom. A boundary face is also part of an adjacent fluid cell, so its velocity is involved in both wall impermeability and the cell's mass balance. Changing that velocity to satisfy one condition can affect the other. These suggest that admissibility should be built into the output itself, rather than imposed afterward.

\begin{figure*}[t]
\vskip -0.2in
\vspace{-10pt}
  \centering
  \includegraphics[width=\textwidth]{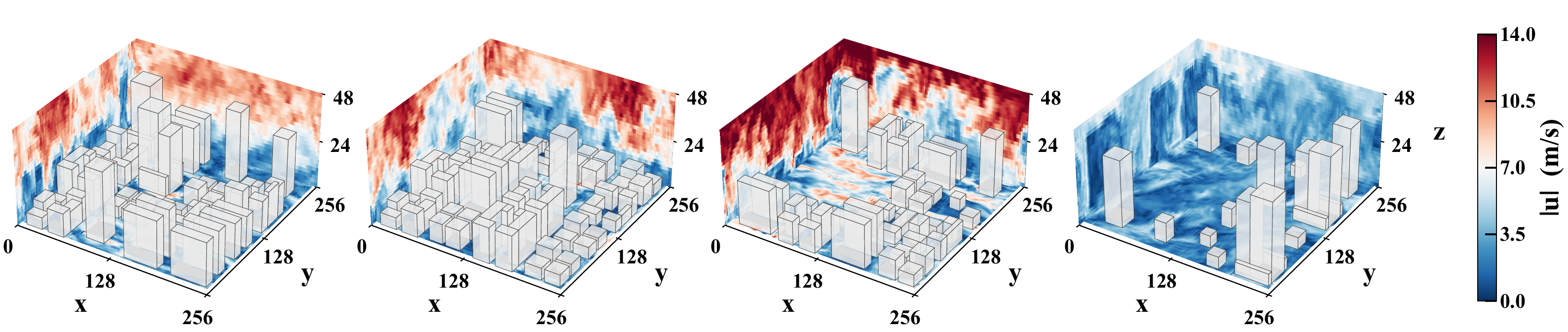}
  \caption{Visualization of the layouts and the instantaneous wind speed $|\mathbf{u}|$ in four \bench{} cases, which differ in both morphology and inflow condition. Axes are in meters.}
  \label{fig:layouts}
  \vspace{-8pt}
\end{figure*}

{\bfseries Urban Wind Field Dataset.}
To systematically study the kinematic admissibility in Definition~\ref{def:ka} across different urban geometries and inflow conditions, we construct \bench{}, a dataset of three-dimensional instantaneous wind fields that resolve buildings on the grid and span a range of layouts and inflow conditions, using the PALM large-eddy simulation model~\citep{maronga2020overview}. It contains 85 building-resolving wind-field sequences that vary in both morphology and inflow condition. Each sequence consists of 100 consecutive snapshots (frames) at $\Delta t=1$\,s after the flow has become statistically stationary. Fig.~\ref{fig:layouts} shows representative layouts and wind fields. Every snapshot satisfies the two kinematic conditions within numerical precision.

\subsection{Problem Formulation} \label{sec:bench}

{\bfseries Task Definition.}

Considering the temporal nature of wind fields, we formulate a history-based prediction task that incorporates the building layout and inflow condition under which the wind field evolves, with the latter including wind direction, free-stream speed, and atmospheric stability. 
Given a set of sequences $\mathcal{D}$, the $n$ most recent snapshots, the layout $g$ and the inflow condition $a$, a predictor $f_\omega$ with parameters $\omega$ produces the next frame as $\hat{\mathbf{u}}^{t+1} = f_\omega(\mathbf{u}^{t-n+1},\dots,\mathbf{u}^{t};\,g,\,a)$, and a $k$-step prediction $\hat{\mathbf{u}}^{t+k}$ is obtained autoregressively. The objective function is as follows:
\begin{equation} \label{eq:pre-objective}
    \omega^{*} = \arg\min_{\omega}\; \mathbb{E}_{(g,a)\sim\mathcal{D}}\;\mathbb{E}_{t,k}\Big[\ell\big(\hat{\mathbf{u}}^{t+k},\, \mathbf{u}^{t+k}\big)\Big],
    \quad \text{s.t.}\;\; \hat{\mathbf{u}}^{t+k}\in\mathcal{V}(g), \;\;\; 
\end{equation}
where $\ell (\cdot)$ is a loss function, $\mathbf{u}^{t+k}$ is the ground-truth frame at horizon $k$, and $t$ ranges over the frames with $t\geq n$. Eq.~(\ref{eq:pre-objective}) requires the admissibility constraint to hold for any $\omega$, layout, inflow condition, and at every rollout horizon. Beyond admissibility, the predictor should generalize across variations in urban geometry and inflow conditions, thus motivating three evaluation settings: unseen layout instances, unseen morphology classes, and held-out inflow combinations.

{\bfseries Making Admissibility a Property of the Architecture.}
Existing approaches typically impose admissibility on an unconstrained prediction, for example through a soft penalty $\lambda|\mathcal{C}(g)\hat{\mathbf{u}}|^{2}$~\citep{raissi2019physics,karniadakis2021physics}, leaving the residual to the optimizer, or through pressure projection~\citep{li2026project}, which maps each prediction to $\mathcal{V}(g)$ but requires a geometry-dependent elliptic solve at every step. A structured output representation~\citep{ciaremalla2024variable} instead parameterizes predictions directly within the admissible space through a map:

\begin{equation} \label{eq:pre-nullspace}
    \Psi_g:\; \boldsymbol{\theta} \longmapsto \hat{\mathbf{u}},
    \qquad \boldsymbol{\theta}\in\mathbb{R}^{P}~\text{unconstrained},
\end{equation}
where $\boldsymbol{\theta}$ contains $P$ unconstrained variables. If
$\operatorname{range}(\Psi_g)\subseteq\mathcal{V}(g)$, every output of $\Psi_g$ is kinematically admissible regardless of $\boldsymbol{\theta}$. A direct construction of $\Psi_g$ is to use a basis of $\ker\mathcal{C}(g)$. However, such a basis must be factorized separately for each layout and has no natural correspondence across different geometries, making it difficult to share a single predictor across layouts. We instead construct $\Psi_g$ as a nested map using two discrete identities of the staggered grid: a discrete curl has no divergence, and a discrete gradient has no curl. 
Specifically, we define the update $\Delta\mathbf{u}$ as the curl of a volume vector potential, which conserves mass in every cell regardless of the mask. The mask is therefore needed only for the second condition. Impermeability is enforced within the same operator by assembling the potential on the edges bounding $\mathcal{F}_{\partial}(g)$ from a shared scalar potential, with these boundary edges providing the layout information for $\Psi_g$. Since $\mathcal{V}(g)$ is a linear subspace, adding an admissible update to an admissible frame preserves admissibility by the autoregressive rollout.

%% file: latex/chapters/3_method.tex
\section{Methodology}

\begin{figure*}[t]
  \centering
  \includegraphics[width=\textwidth]{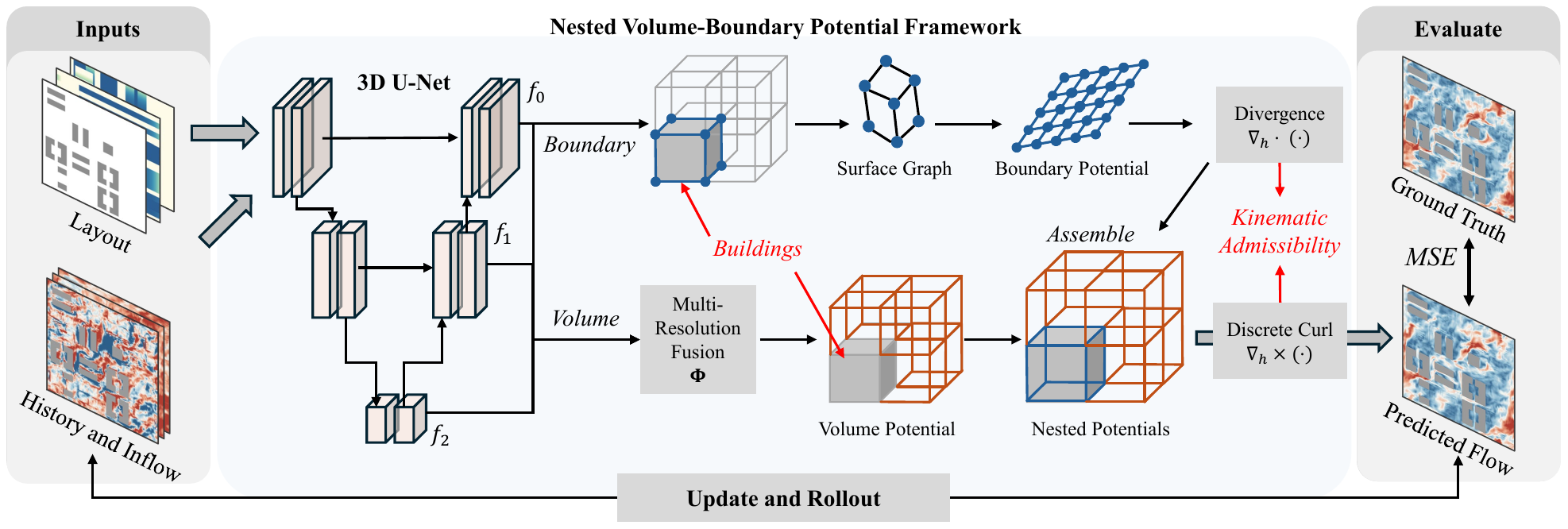}
  \caption{Architecture overview. Given the layout, flow history, and inflow condition, \method{} consists of two parts: a volume potential module that captures multi-scale circulation on the volume grid, and a boundary potential module that models geometry-dependent boundary interactions on the surface mesh. Their outputs are assembled into a nested potential, whose discrete curl yields a kinematically admissible velocity increment for autoregressive rollout.}
  \label{fig:framework}
  \vskip -0.2in
\end{figure*}

As shown in Fig.~\ref{fig:framework}, \method{} emits two potential fields and reads the velocity update from them. A volumetric module produces the volume vector potential on the grid edges, a surface graph network produces the boundary scalar potential on the surface vertices. Then a parameter-free assembly operator combines them and applies the discrete curl. The following sections detail their formulations.

\subsection{Nested Volume-boundary Potential Framework} \label{sec:rep}

\paragraph{Volume Vector Potential.}
Rather than directly predicting the velocity on faces, we follow prior divergence-free parameterizations~\citep{richter2022neural,wandel2021teaching} and predict a vector potential on grid edges. Let $h$ be the grid spacing, $\mathcal{E}$ be the set of grid edges, and $A_e\in\mathbb{R}$ be the value carried by edge $e\in\mathcal{E}$, the collection of which forms the volume vector potential $\mathbf{A}$. Following the discrete Stokes theorem~\citep{hyman1997natural, desbrun2005discrete}, the velocity update on a face $f$ is recovered as the circulation of $\mathbf{A}$ around its boundary via
\begin{equation} \label{eq:curl}
    \Delta u_f \;=\; (\nabla_h\times\mathbf{A})_f \;=\; \frac{1}{h}\sum_{e\in\partial f}\sigma_{f,e}\,A_e ,
\end{equation}
where $\partial f$ denotes the four edges bounding $f$ with a fixed orientation and
$\sigma_{f,e}\in\{+1,-1\}$ is the incidence sign of $e$ in that oriented boundary. Local mass conservation is guaranteed automatically, as formalized in the following proposition. The proof is provided in Appendix~\ref{app:constructions}.
\begin{proposition} \label{prop:div}
For every $\mathbf{A}\in\mathbb{R}^{|\mathcal{E}|}$, the field defined by Eq.~(\ref{eq:curl}) satisfies $(\nabla_h\!\cdot\!\Delta\mathbf{u})_c=0$ in every cell $c$ and in exact arithmetic.
\end{proposition}


\paragraph{Shared Scalar Potential on the Boundary.}
The curl construction above guarantees local mass conservation for any $\mathbf{A}$, but wall impermeability further constrains its values on the boundary edges. Let $\mathcal{E}_{\partial}(g)\subset\mathcal{E}$ denote the set of boundary edges, including those incident to at least one solid cell and those on the ground and lid. These edges span a two-dimensional complex $\mathcal{B}$ with vertex set $\mathcal{N}$, whose elements are denoted by $\nu$. Every face in $\mathcal{F}_{\partial}(g)$ is bounded by four edges in $\mathcal{E}_{\partial}(g)$, so impermeability depends only on $\mathbf{A}|_{\mathcal{E}_{\partial}}$. In contrast to methods that impose wall impermeability directly on the predicted velocity~\citep{sukumar2022exact,higakienhancing}, the condition is enforced one degree down the discrete de Rham complex, as formalized in the following proposition.


\begin{proposition} \label{prop:exact}
Let $\{\chi^{(1)},\dots,\chi^{(\beta_1)}\}$ be a harmonic basis of a component of
$\mathcal{B}$ with first Betti number $\beta_1$, and let the boundary values of
$\mathbf{A}$ be taken in the form
\begin{equation} \label{eq:hodge}
    A_e \;=\; \frac{G_{e^{+}}-G_{e^{-}}}{h} \;+\; \sum_{m=1}^{\beta_1} \gamma_m\, \chi^{(m)}_e ,
    \qquad e\in\mathcal{E}_{\partial}(g) ,
\end{equation}
for an arbitrary $G:\mathcal{N}\to\mathbb{R}$ and arbitrary coefficients $\gamma_m$.
Then $\Delta u_f=0$ for every $f\in\mathcal{F}_{\partial}(g)$.
\end{proposition}
Here, $G:\mathcal{N}\to\mathbb{R}$ is the shared scalar potential on the boundary vertices, $e^{+}$ and $e^{-}$ are the two endpoints of edge $e$, and $\gamma_m$ are the coefficients of the global circulation modes.
Proposition~\ref{prop:exact} turns impermeability into a recipe for the boundary values of $\mathbf{A}$: take the three Cartesian components on $\mathcal{E}_{\partial}(g)$ from the differences of one and the same scalar field, so that the circulation around every boundary face telescopes to zero. Predicting $G$ is then a problem posed on $\mathcal{B}$ rather than on the grid.

\paragraph{Assembly Operator.}
The two modules provide the interior values of $\mathbf{A}$, the boundary potential $G$, and the coefficients $\gamma_m$. The assembly operator retains the predicted values on the interior edges and replaces those on $\mathcal{E}_{\partial}(g)$ with Eq.~(\ref{eq:hodge}). The discrete curl in Eq.~(\ref{eq:curl}) is then applied to the assembled potential to obtain the velocity update $\Delta\mathbf{u}$. The layout enters $\Psi_g$ only through the boundary-edge set $\mathcal{E}_{\partial}(g)$, while the interior edges remain unconstrained.
We minimize the mean squared error between the updated field and the ground truth over the fluid cell, with no penalty on either condition of Eq.~(\ref{eq:pre-admissible}) because both hold for every value of the assembled vector potential. For the stability of autoregressive process, we introduce the pushforward scheme~\citep{brandstetter2022message}, whose perturbed inputs are predictions of the same map and are therefore admissible as well. 
The next two subsections describe in detail how the volume potential and the boundary potential are predicted.

\subsection{Multiresolution Volume Potential} \label{sec:volume}
\paragraph{Inputs.}
The volumetric module is conditioned on the $n$ most recent frames, the building layout and a single inflow scalar, which together form the channels of one input tensor. Each frame supplies its three velocity components. The layout supplies the occupancy mask and four horizontal distance fields, one per horizontal direction, each giving at every fluid cell the distance to the nearest solid cell along that direction, with periodic wrap-around in $x$ and $y$. The inflow supplies the surface heat flux broadcast over the grid, which carries the stratification; the direction and the speed are already visible in the frames. A three-dimensional U-Net encodes this tensor into feature maps at several resolutions.

\paragraph{Multiresolution Branches.}
Given a U-Net with $S$ resolution levels, we denote $\mathbf{f}_s$ for the features it computes at level $s$, which carry $C_s$ channels on the grid coarsened by $2^{s}$. Typically, the finest layer $\mathbf{f}_0$ serves as the output, and the deeper feature maps serve only as intermediates. We take a candidate potential from each level instead, so that every level reaches the output: 
\begin{equation} \label{eq:multires}
    \mathbf{a}_s \;=\; \mathcal{C}^{\mathrm{out}}_s\big(\mathcal{U}_s(\mathcal{C}^{\mathrm{in}}_s(\mathbf{f}_s))\big),
    s = 0,\dots,S-1 ,
\end{equation}
where $\mathcal{C}^{\mathrm{in}}_s$ is a convolution that mixes $\mathbf{f}_s$ at level $s$ and reduces its $C_s$ channels to a common width $c_{\mathcal{A}}$, $\mathcal{U}_s$ prolongates the result to the grid, and
$\mathcal{C}^{\mathrm{out}}_s$ is a convolution taking those $c_{\mathcal{A}}$ channels to the three Cartesian components. Both $\mathcal{C}^{\mathrm{in}}_0$ and $\mathcal{U}_0$ are the identity. The arrangement is that of a feature pyramid~\citep{lin2017feature}. A learned matrix $\Phi\in\mathbb{R}^{3\times 3S}$ applied at every grid point combines the $S$
candidates. We denote $[\,\cdot\,;\,\cdot\,]$ as concatenation, the volume vector potential $\mathbf{A}$ can be obtained by
\begin{equation} \label{eq:fuse}
    \mathbf{A} \;=\; \Phi\big(\left[\,\mathbf{a}_0\,;\,\dots\,;\,\mathbf{a}_{S-1}\,\right]\big).
\end{equation}

\begin{wrapfigure}{R}{0.5\textwidth}
  \vspace{-12pt}
  \centering
  \includegraphics[width=1\linewidth]{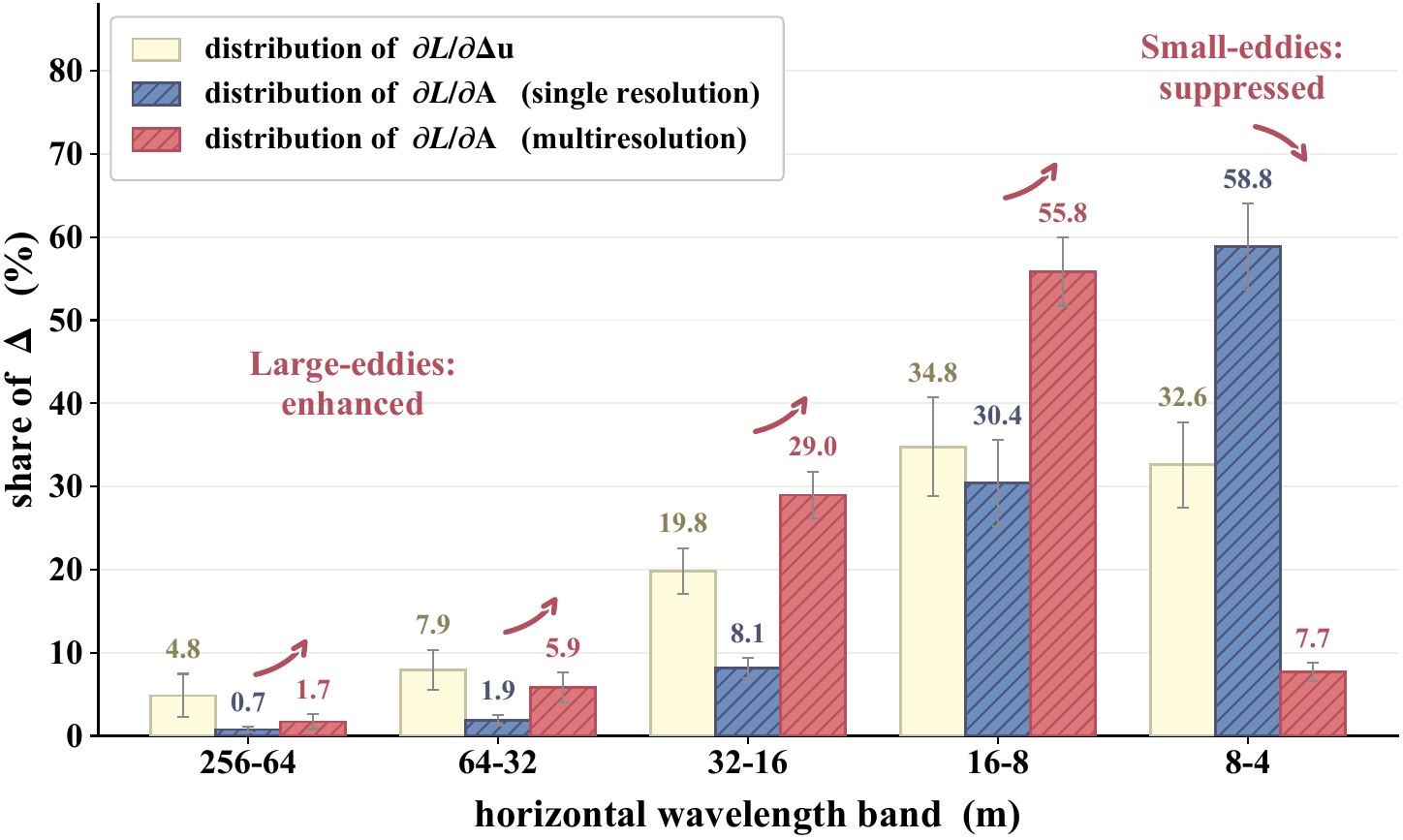}
  \caption{
  Share of the gradient in each horizontal wavelength band: $\Delta\mathbf{u}$ (left), single-resolution $\mathbf{A}$ (middle), and multiresolution $\mathbf{A}$ (right). The multiresolution read-out raises the share carried by the Large-eddies and lowers the share carried by the small ones, undoing part of the shift that backpropagation through the curl introduces.
  }
  \label{fig:bands}
  \vspace{-9pt}
\end{wrapfigure}

\paragraph{Why Multiresolution?}
\method{} predicts vector potential rather than the update directly. Gradients travel back to the potential through the adjoint of the curl, and that adjoint is not neutral: it multiplies each Fourier component of the gradient by a gain proportional to its wavenumber. The long wavelengths arrive attenuated and the short ones amplified, compared with the gradient the same loss would hand to a directly predicted velocity. Across the wavelengths the grid resolves, the two ends of that gain differ by a factor of order the domain size divided by the grid spacing. We provide the proof in Appendix~\ref{app:adjoint}. The effect is that different wavelengths learn at very different speeds. Producing the potential from the output of the U-Net alone, at the full resolution, every weight of the volumetric module is trained through this one reweighted gradient and no weight is dedicated to the long wavelengths. Those wavelengths carry the large-scale structure of the flow, and they are the part of the error that improves most slowly. Nothing prevents the parameterization from representing them; what suffers is the accuracy reached within a fixed training budget. 
Fig.~\ref{fig:bands} shows how the gradient that $\Delta\mathbf{u}$ and two type of $\mathbf{A}$ receive during backpropagation distributes over the wavelength bands.

\subsection{Surface Graph Network for the Boundary Potential} \label{sec:surface}

With these volumetric features, \method{} next explores the shared scalar potential $G$ on the boundary. The boundary potential lives on $\mathcal{B}$, where a three-dimensional convolution is a poor inductive bias: two fa\c{c}ades across a narrow street are neighbours in Euclidean distance, while the path between them on $\mathcal{B}$ descends one fa\c{c}ade, crosses the ground and ascends the other. We therefore predict $G$ with a network whose spatial operator is an aggregation along $\mathcal{E}_{\partial}(g)$. 

\paragraph{Learning Representations on the Boundary.}
We regard $\mathcal{B}$ as a graph with node set $\mathcal{N}$ and edge set $\mathcal{E}_{\partial}(g)$. Its nodes lie on a regular lattice, so each has at most six neighbors, one per axis direction $\delta\in\{1,\dots,6\}$, and an indicator $o_{\nu,\delta}\in\{0,1\}$ records whether that neighbour exists. Every node carries a static descriptor $\mathbf{d}_\nu\in\mathbb{R}^{d_{\mathrm{des}}}$ that collects the six indicators, the height of the node normalized by the domain height, and a slip fraction $\eta_\nu\in[0,1]$ that is zero on a no-slip surface and one on a free-slip one, obtained from the aerodynamic roughness length. 
Using the full-resolution output $\mathbf{f}_0$ of the volumetric module, we can obtain a representation of each cell. Then we introduce a learnable aggregation $\mathbf{p}_\nu\in\mathbb{R}^{8}$ to map cells to central nodes within $\mathcal{B}$. Let $\mathbf{X}_\nu\in\mathbb{R}^{8\times C_0}$ stack the eight rows $\mathbf{f}_0[c_j(\nu)]$ of the finest feature map of Section~\ref{sec:volume}, taken at the eight cells incident to $\nu$ in slots whose order encodes direction, and let $\mathbf{1}_\nu\in\{0,1\}^{8}$ mark the slots inside the domain. The node feature $\mathbf{F}_\nu\in\mathbb{R}^{C_0}$ can be obtained by:
\begin{equation} \label{eq:gather}
    \mathbf{p}_\nu \;=\; \mathrm{softmax}\big(\mathrm{MLP}(\mathbf{X}_\nu,\mathbf{d}_\nu) + \log\mathbf{1}_\nu\big),
    \qquad
    \mathbf{F}_\nu \;=\; \mathbf{X}_\nu^{\!\top}\mathbf{p}_\nu .
\end{equation}
The node features are then refined on $\mathcal{B}$. A graph module of $K$ layers starts from
$\mathbf{h}^{(0)}_\nu = \mathbf{W}_{\mathrm{in}}[\,\mathbf{F}_\nu\,;\,\mathbf{d}_\nu\,]$, and layer
$\tau$ updates every node by:
\begin{equation} \label{eq:gnn}
    \mathbf{h}^{(\tau)}_\nu \;=\; \mathbf{h}^{(\tau-1)}_\nu \;+\; \mathrm{SiLU}\Big(\mathbf{U}^{(\tau)}\big[\;\mathbf{h}^{(\tau-1)}_\nu \;;\; \sum_{\delta=1}^{6} o_{\nu,\delta}\,\mathbf{W}^{(\tau)}_{\delta}\,\mathbf{h}^{(\tau-1)}_{\mathrm{nbr}(\nu,\delta)}\;\big]\Big) ,
\end{equation}
where $\mathrm{nbr}(\nu,\delta)$ is the neighbour of $\nu$ in direction $\delta$. $\mathbf{W}_{\mathrm{in}}\in\mathbb{R}^{c_{\mathcal{B}}\times(C_0+d_{\mathrm{des}})}$, $\mathbf{W}^{(\tau)}_{\delta}\in\mathbb{R}^{c_{\mathcal{B}}\times c_{\mathcal{B}}}$ and $\mathbf{U}^{(\tau)}\in\mathbb{R}^{c_{\mathcal{B}}\times 2c_{\mathcal{B}}}$ are the learnable weights. Each direction carries its own $\mathbf{W}^{(\tau)}_{\delta}$, since a step along a fa\c{c}ade and a step around a corner are distinct operations on $\mathcal{B}$. We denote $\mathbf{H}\in\mathbb{R}^{|\mathcal{N}|\times c_{\mathcal{B}}}$ for the final states $\mathbf{h}^{(K)}_\nu$ stacked over the nodes.

\paragraph{Multiresolution Readout.} The reweighting of Section~\ref{sec:volume} reaches $G$ through a curl followed by a gradient, so its long wavelengths sit one power of the gain further from the loss than those of $\mathbf{A}$, as Remark~\ref{rem:gaintwice} makes precise. We therefore read $G$ out at several levels, as we do for $\mathbf{A}$:
\begin{equation} \label{eq:gread}
    g_s \;=\; \mathrm{pool}_s\big(\mathbf{H}\mathbf{w}_s\big),
    \qquad s = 0,\dots,S-1 ,
\end{equation}
where $\mathbf{w}_s\in\mathbb{R}^{c_{\mathcal{B}}}$ reads one value from each node, and $\mathrm{pool}_s$ averages it within the corresponding level-$s$ block followed by parameter-free neighbor averaging along $\mathcal{E}_{\partial}(g)$. The $S$ readouts are combined as in Eq.~(\ref{eq:fuse}). Since adding a component-wise constant to $G$ leaves the velocity update unchanged, we subtract the component-wise mean from the predicted boundary potential.

\paragraph{Harmonic Completion.} We specialize the harmonic part of Proposition~\ref{prop:exact} to our urban domain. The number of harmonic coefficients is determined by the topology of each component of $\mathcal{B}$. Here, $\mathcal{B}$ consists of the ground, the lid, together with the building surfaces. The lid is periodic in both horizontal directions and is therefore topologically a torus with first Betti number $\beta_1=2$, giving two harmonic coefficients, $\gamma_1$ and $\gamma_2$, associated with the two periodic directions. These modes are necessary because the gradient of $G$ has zero circulation around every closed loop and cannot represent these global circulations. In our setting, they control the mean horizontal wind. Summing Eq.~(\ref{eq:curl}) along a vertical column cancels the vertical differences in pairs, while the horizontal differences vanish under periodic averaging, giving
\begin{equation} \label{eq:meanflow}
    \langle \Delta u\rangle_{\mathrm{dom}} = -\frac{\langle A^y\rangle_{\mathrm{lid}}}{L_z},
    \qquad
    \langle \Delta v\rangle_{\mathrm{dom}} = +\frac{\langle A^x\rangle_{\mathrm{lid}}}{L_z} ,
\end{equation}
where $\langle\cdot\rangle_{\mathrm{dom}}$ and $\langle\cdot\rangle_{\mathrm{lid}}$ average over the domain and over the top layer of edges, and $L_z$ is the domain height. Setting $\gamma_1=\gamma_2=0$ would force both means to vanish and freeze the mean wind at the value of the first frame. We hence take them from the lid means of the horizontal components the volumetric module emits, and set those of the ground component to zero.
Substituting $G$ and the harmonic coefficients into Eq.~(\ref{eq:hodge}) gives the boundary values $\mathbf{A}|_{\mathcal{E}_{\partial}(g)}$. These values replace those produced by Eq.~(\ref{eq:fuse}) on $\mathcal{E}_{\partial}(g)$, while the interior values of $\mathbf{A}$ are retained. The assembly operator of Section~\ref{sec:rep} then applies the discrete curl once to the merged potential. By Propositions~\ref{prop:div} and~\ref{prop:exact}, the resulting update satisfies local mass conservation and wall impermeability for every layout and every value of $\boldsymbol{\theta}$. Finally, we can obtain the kinematically admissible wind velocity based on historical frames.

%% file: latex/chapters/4_experiment.tex
\section{Experiments}
\label{sec:exp}
In this section we evaluate \method{} on \bench{} by autoregressive wind-field prediction tasks. We report the mean squared error (MSE) of the predicted velocity over the fluid cells at several rollout horizons, and report two physical residuals to measure the kinematic admissibility: the maximum divergence $\max|\nabla\!\cdot\!\mathbf{u}|$, which measures whether mass is conserved, and the maximum normal velocity at solid surfaces $\max|\mathbf{u}\!\cdot\!\mathbf{n}|$, which measures whether air is allowed to pass through boundaries.
The two physical residuals are worst cases, taken over all test sequences and all four horizons. Details and supplementary experiments are provided in Appendix~\ref{app:data},~\ref{app:exp}, and~\ref{app:results}.

\begin{table}[tb]
\begin{minipage}{0.49\textwidth}
\centering
\vspace{-4mm}
\caption{Performance on the \textbf{unseen layouts} setting of \bench{}. MSE is reported at four rollout horizons. Highlighted are the top \textcolor{mydeepred}{first}, \textcolor{mydeepblue}{second} results. `$<$' in two constraint columns means the power-of-ten bound on the measured value, and $0$ is exact.}
\label{tab:id}
\renewcommand{\arraystretch}{1.2}
\resizebox{1\linewidth}{!}{
\begin{tabular}{l|cccc|cc}
    \toprule
    \multirow{2}{*}{Methods} & \multicolumn{4}{c|}{MSE$\downarrow$ ($\times 10^{-3}$)} & \multicolumn{2}{c}{worst over all $k$} \\\cmidrule{2-5}\cmidrule{6-7}
    & \multicolumn{1}{c}{$k{=}1$} & \multicolumn{1}{c}{$k{=}4$} & \multicolumn{1}{c}{$k{=}8$} & \multicolumn{1}{c|}{$k{=}16$} & \multicolumn{1}{c}{$\max|\nabla\!\cdot\!\mathbf{u}|\downarrow$} & \multicolumn{1}{c}{$\max|\mathbf{u}\!\cdot\!\mathbf{n}|\downarrow$} \\
    \midrule
    DMD & 55.27\small{$\pm$17.53} & 55.40\small{$\pm$17.66} & 63.10\small{$\pm$19.82} & 65.58\small{$\pm$20.46} & $<10^{-6}$ & $0$ \\
    Time-Mean & 29.84\small{$\pm$9.59} & 46.59\small{$\pm$14.63} & 52.06\small{$\pm$16.13} & 53.98\small{$\pm$16.61} & $<10^{-6}$ & $0$ \\
    POD & 54.75\small{$\pm$16.93} & 54.69\small{$\pm$16.92} & 54.87\small{$\pm$16.95} & 55.36\small{$\pm$17.21} & $<10^{-5}$ & 1.70 \\
    Dil-ResNet & 0.85\small{$\pm$0.43} & 8.36\small{$\pm$3.16} & 29.37\small{$\pm$8.94} & 69.03\small{$\pm$15.56} & 0.46 & 1.10 \\
    TFNO & 21.10\small{$\pm$7.11} & 35.49\small{$\pm$11.25} & 53.42\small{$\pm$14.94} & 91.05\small{$\pm$17.71} & 1.14 & 1.84 \\
    AFNO & 10.44\small{$\pm$3.71} & 28.27\small{$\pm$9.52} & 42.05\small{$\pm$13.41} & 60.85\small{$\pm$18.87} & 2.02 & 2.92 \\
    AViT & 7.69\small{$\pm$2.63} & 17.60\small{$\pm$6.30} & 32.47\small{$\pm$11.01} & 58.01\small{$\pm$17.46} & 1.52 & 1.64 \\
    P3D & 1.00\small{$\pm$0.49} & 5.85\small{$\pm$2.69} & 16.79\small{$\pm$7.48} & 43.44\small{$\pm$16.47} & 0.44 & 0.87 \\
    U-Net & \textcolor{mydeepblue}{0.61\small{$\pm$0.30}} & \textcolor{mydeepblue}{3.30\small{$\pm$1.60}} & \textcolor{mydeepblue}{9.06\small{$\pm$4.21}} & \textcolor{mydeepblue}{25.94\small{$\pm$10.32}} & 0.16 & 1.06 \\
    \hline\hline
    \textbf{\method{}} & \textcolor{mydeepred}{\textbf{0.22\small{$\pm$0.13}}} & \textcolor{mydeepred}{\textbf{1.07\small{$\pm$0.55}}} & \textcolor{mydeepred}{\textbf{3.89\small{$\pm$2.02}}} & \textcolor{mydeepred}{\textbf{15.91\small{$\pm$7.52}}} & $<10^{-6}$ & $<10^{-6}$ \\
    \bottomrule
\end{tabular}
}
\end{minipage}
\hfill
\begin{minipage}{0.49\textwidth}
\centering
\vspace{-4mm}
\caption{Performance on the \textbf{unseen morphologies} setting of \bench{}. MSE is reported at four rollout horizons. Highlighted are the top \textcolor{mydeepred}{first}, \textcolor{mydeepblue}{second} results. `$<$' in two constraint columns means the power-of-ten bound on the measured value, and $0$ is exact.}
\label{tab:ood1}
\renewcommand{\arraystretch}{1.2}
\resizebox{\linewidth}{!}{
\begin{tabular}{l|cccc|cc}
    \toprule
    \multirow{2}{*}{Methods} & \multicolumn{4}{c|}{MSE$\downarrow$ ($\times 10^{-3}$)} & \multicolumn{2}{c}{worst over all $k$} \\\cmidrule{2-5}\cmidrule{6-7}
    & \multicolumn{1}{c}{$k{=}1$} & \multicolumn{1}{c}{$k{=}4$} & \multicolumn{1}{c}{$k{=}8$} & \multicolumn{1}{c|}{$k{=}16$} & \multicolumn{1}{c}{$\max|\nabla\!\cdot\!\mathbf{u}|\downarrow$} & \multicolumn{1}{c}{$\max|\mathbf{u}\!\cdot\!\mathbf{n}|\downarrow$} \\
    \midrule
    DMD & 68.62\small{$\pm$17.04} & 68.61\small{$\pm$17.11} & 79.65\small{$\pm$17.42} & 79.47\small{$\pm$15.27} & $<10^{-6}$ & $0$ \\
    Time-Mean & 36.77\small{$\pm$9.47} & 58.37\small{$\pm$13.67} & 66.23\small{$\pm$13.28} & 65.53\small{$\pm$11.70} & $<10^{-6}$ & $0$ \\
    POD & 80.82\small{$\pm$21.94} & 80.97\small{$\pm$22.11} & 81.72\small{$\pm$22.41} & 83.71\small{$\pm$22.99} & $<10^{-5}$ & 1.86 \\
    Dil-ResNet & 1.30\small{$\pm$0.71} & 12.24\small{$\pm$4.93} & 39.94\small{$\pm$11.54} & 86.07\small{$\pm$17.04} & 0.47 & 1.15 \\
    TFNO & 27.31\small{$\pm$8.10} & 46.72\small{$\pm$12.84} & 66.60\small{$\pm$15.25} & 102.29\small{$\pm$16.90} & 1.22 & 2.03 \\
    AFNO & 13.89\small{$\pm$4.84} & 37.66\small{$\pm$11.05} & 57.53\small{$\pm$14.49} & 86.32\small{$\pm$21.09} & 3.28 & 4.62 \\
    AViT & 10.58\small{$\pm$3.46} & 27.09\small{$\pm$8.50} & 51.49\small{$\pm$13.86} & 86.27\small{$\pm$18.79} & 2.09 & 1.53 \\
    P3D & 1.58\small{$\pm$0.84} & 9.25\small{$\pm$4.42} & 26.45\small{$\pm$10.55} & 63.19\small{$\pm$16.99} & 1.37 & 0.99 \\
    U-Net & \textcolor{mydeepblue}{0.94\small{$\pm$0.52}} & \textcolor{mydeepblue}{5.57\small{$\pm$2.96}} & \textcolor{mydeepblue}{15.52\small{$\pm$7.56}} & \textcolor{mydeepblue}{41.33\small{$\pm$15.35}} & 0.20 & 1.05 \\
    \hline\hline
    \textbf{\method{}} & \textcolor{mydeepred}{\textbf{0.32\small{$\pm$0.21}}} & \textcolor{mydeepred}{\textbf{1.85\small{$\pm$1.03}}} & \textcolor{mydeepred}{\textbf{6.61\small{$\pm$3.34}}} & \textcolor{mydeepred}{\textbf{25.10\small{$\pm$9.88}}} & $<10^{-6}$ & $<10^{-6}$ \\
    \bottomrule
\end{tabular}
}
\end{minipage}
\end{table}

\subsection{Experiment Setup}
\label{exp:setting}

\paragraph{Evaluation Settings.}

\begin{table}[t]
\vskip -0.1in
\centering
\caption{Statistics of three settings on \bench{}.}
\label{tab:splits}
\small
\begin{tabular}{lccccc}
\toprule
Setting & \#layouts & building density & \#inflows & \#snapshots & \#samples \\
\midrule
Unseen layouts      & 30 & 0.17--0.38 & 1 & 3,000 & 1,908 \\
Unseen morphologies & 34 & 0.11--0.38 & 1 & 3,400 & 1,980 \\
Unseen inflows      &  5 & 0.11--0.38 & 9 & 4,500 & 2,520 \\
\bottomrule
\end{tabular}
\vskip -0.2in
\end{table}

We construct three evaluation settings on \bench{}. \textbf{Unseen layouts} uses 1,608 / 192 samples for training / validation under one inflow, while the 108 samples on the 6 held-out layouts (same morphology) are used for testing. \textbf{Unseen morphologies} uses 1,608 / 192 samples for training / validation under one inflow, while the 180 samples on the 10 layouts (unseen morphology) are used for testing. \textbf{Unseen inflows} uses 2,010 / 240 samples for training / validation under 9 inflow conditions over wind direction, free-stream speed and stability, while the 270 samples on the 3 held-out conditions are used for testing. Every test sequence is rolled out from the same 18 start frames, and each MSE entry reports the mean and standard deviation over these rollouts. The statistics of every setting are given in Table~\ref{tab:splits}.

\paragraph{Baselines.}
We compare our method with nine wind-field predicting methods, which can be divided into three groups:
\begin{enumerate}[0]
\item[$\bullet$] Reduced-order methods include the temporal mean of the input window (Time-Mean), dynamic mode decomposition (DMD)~\citep{schmid2010dynamic,tu2013dynamic}, and proper orthogonal decomposition (POD)~\citep{berkooz1993proper}. They predict from linear combinations of observed snapshots.
\item[$\bullet$] Convolutional surrogates include the dilated residual network (Dil-ResNet)~\citep{stachenfeld2021learned}, the pseudo-3D residual network (P3D)~\citep{holzschuh2025p3d}, and a three-level 3D U-Net with ConvNeXt blocks (U-Net)~\citep{ohana2024well}.
\item[$\bullet$] Spectral and attention operators include the tensorized Fourier neural operator (TFNO)~\citep{li2020fourier,kossaifi2023multi}, the adaptive Fourier neural operator (AFNO)~\citep{guibas2021adaptive}, and the axial vision transformer (AViT)~\citep{ohana2024well}.
\end{enumerate}

\subsection{Experimental Results}
\label{main}

\paragraph{Main Results.}
Tables~\ref{tab:id} and~\ref{tab:ood1} compare all ten methods on the unseen layout and the unseen morphology settings, and Table~\ref{tab:ood2} in Appendix~\ref{app:main_sup} reports the unseen inflow setting. Against the strongest baseline, \method{} lowers the error at $k{=}16$ by 29\% averaging on the three settings, and by 61.9\% at $k{=}1$, so the relative margin is widest on the short horizons. Time-Mean and DMD build their prediction from historical snapshots and therefore inherit exact physical residuals, but they carry no dynamics and perform poorly throughout. POD draws its basis from the training sequences whose impermeable boundaries lie elsewhere. U-Net is the most accurate among learnable baselines, yet its worst face lets air through at 1.06, 1.05 and 1.30 times the free-stream speed, more than five orders of magnitude above \method{}. AFNO diverges outright under inflow shift, to an MSE of 1.6${\times}10^{3}$ with a divergence of 5.9${\times}10^{2}$. The \textbf{unseen inflows} setting is the hardest for every learned model, to the point that POD overtakes all of them at $k{=}16$, while \method{} still leads by 11.7\%. In summary, \method{} consistently outperforms both reduced-order methods and neural surrogates across all settings, demonstrating the nested potential framework's effectiveness in accuracy and kinematic admissibility.

\begin{table}[t]
\begin{minipage}{0.49\textwidth}
\centering
\vspace{-4mm}
\caption{Ablation on the \textbf{unseen layouts} setting of \bench{}. MSE is reported at two rollout horizons. Highlighted are the top \textcolor{mydeepred}{first}, \textcolor{mydeepblue}{second} results. `$<$' in two constraint columns means the power-of-ten bound on the measured value, and $0$ is exact.}
\label{tab:ablation}
\renewcommand{\arraystretch}{1.15}
\resizebox{\textwidth}{!}{
\begin{tabular}{l|cc|cc|cc}
    \toprule
    \multirow{2}{*}{Variant} & \multicolumn{2}{c|}{MSE$\downarrow$} & \multicolumn{2}{c|}{$\max|\nabla\!\cdot\!\mathbf{u}|\downarrow$} & \multicolumn{2}{c}{$\max|\mathbf{u}\!\cdot\!\mathbf{n}|\downarrow$} \\\cmidrule{2-7}
    & \multicolumn{1}{c}{$k{=}1$} & \multicolumn{1}{c|}{$k{=}16$} & \multicolumn{1}{c}{$k{=}1$} & \multicolumn{1}{c|}{$k{=}16$} & \multicolumn{1}{c}{$k{=}1$} & \multicolumn{1}{c}{$k{=}16$} \\
    \midrule
    \textbf{\method{}} & \textcolor{mydeepred}{\textbf{0.22\small{$\pm$0.13}}} & \textcolor{mydeepred}{\textbf{15.91\small{$\pm$7.52}}} & $<10^{-6}$ & $<10^{-6}$ & $<10^{-7}$ & $<10^{-6}$ \\
    \hline\hline
    w/o null-space $\Psi_g$ & 0.43\small{$\pm$0.24} & 31.03\small{$\pm$12.05} & 0.20 & 1.24 & 1.58 & 2.00 \\
    w/ soft penalty & 25.00\small{$\pm$4.58} & 112.95\small{$\pm$10.40} & 0.14 & 0.36 & 0.76 & 1.02 \\
    w/ post-projection & 0.38\small{$\pm$0.22} & 25.16\small{$\pm$9.81} & $<10^{-6}$ & $<10^{-6}$ & $0$ & $0$ \\
    w/o boundary $G$ & 0.27\small{$\pm$0.18} & 19.94\small{$\pm$8.29} & $<10^{-6}$ & $<10^{-6}$ & 0.71 & 0.80 \\
    w/o increment $\Delta\mathbf{u}$ & 13.63\small{$\pm$4.61} & 2867.98\small{$\pm$265.31} & $<10^{-6}$ & $<10^{-5}$ & $<10^{-6}$ & $<10^{-5}$ \\
    w/o multi-scale $\mathbf{A}$ & \textcolor{mydeepblue}{0.23\small{$\pm$0.15}} & \textcolor{mydeepblue}{16.20\small{$\pm$6.87}} & $<10^{-6}$ & $<10^{-6}$ & $<10^{-7}$ & $<10^{-6}$ \\
    w/o learned fusion & 0.26\small{$\pm$0.18} & 18.46\small{$\pm$7.98} & $<10^{-6}$ & $<10^{-6}$ & $<10^{-7}$ & $<10^{-6}$ \\
    \bottomrule
\end{tabular}
}
\end{minipage}
\hfill
\begin{minipage}{0.49\textwidth}
\centering
\vspace{-4mm}
\caption{Ablation on the \textbf{unseen morphologies} setting of \bench{}. MSE is reported at two rollout horizons. Highlighted are the top \textcolor{mydeepred}{first}, \textcolor{mydeepblue}{second} results. `$<$' in two constraint columns means the power-of-ten bound on the measured value, and $0$ is exact.}
\label{tab:ablation_ood1}
\renewcommand{\arraystretch}{1.15}
\resizebox{\textwidth}{!}{
\begin{tabular}{l|cc|cc|cc}
    \toprule
    \multirow{2}{*}{Variant} & \multicolumn{2}{c|}{MSE$\downarrow$} & \multicolumn{2}{c|}{$\max|\nabla\!\cdot\!\mathbf{u}|\downarrow$} & \multicolumn{2}{c}{$\max|\mathbf{u}\!\cdot\!\mathbf{n}|\downarrow$} \\\cmidrule{2-7}
    & \multicolumn{1}{c}{$k{=}1$} & \multicolumn{1}{c|}{$k{=}16$} & \multicolumn{1}{c}{$k{=}1$} & \multicolumn{1}{c|}{$k{=}16$} & \multicolumn{1}{c}{$k{=}1$} & \multicolumn{1}{c}{$k{=}16$} \\
    \midrule
    \textbf{\method{}} & \textcolor{mydeepred}{\textbf{0.32\small{$\pm$0.21}}} & \textcolor{mydeepred}{\textbf{25.10\small{$\pm$9.88}}} & $<10^{-6}$ & $<10^{-6}$ & $<10^{-7}$ & $<10^{-6}$ \\
    \hline\hline
    w/o null-space $\Psi_g$ & 0.69\small{$\pm$0.41} & 45.84\small{$\pm$13.89} & 0.23 & 1.03 & 1.69 & 2.20 \\
    w/ soft penalty & 34.05\small{$\pm$2.83} & 150.08\small{$\pm$15.74} & 0.14 & 0.34 & 1.13 & 1.45 \\
    w/ post-projection & 0.59\small{$\pm$0.38} & 38.01\small{$\pm$12.52} & $<10^{-6}$ & $<10^{-6}$ & $0$ & $0$ \\
    w/o boundary $G$ & 0.39\small{$\pm$0.28} & 29.95\small{$\pm$11.58} & $<10^{-6}$ & $<10^{-6}$ & 0.95 & 0.95 \\
    w/o increment $\Delta\mathbf{u}$ & 22.23\small{$\pm$5.58} & 2891.92\small{$\pm$366.57} & $<10^{-6}$ & $<10^{-5}$ & $<10^{-6}$ & $<10^{-5}$ \\
    w/o multi-scale $\mathbf{A}$ & \textcolor{mydeepblue}{0.33\small{$\pm$0.23}} & \textcolor{mydeepblue}{27.30\small{$\pm$10.89}} & $<10^{-6}$ & $<10^{-6}$ & $<10^{-7}$ & $<10^{-6}$ \\
    w/o learned fusion & 0.37\small{$\pm$0.27} & 29.03\small{$\pm$11.52} & $<10^{-6}$ & $<10^{-6}$ & $<10^{-7}$ & $<10^{-6}$ \\
    \bottomrule
\end{tabular}
}
\end{minipage}
\end{table}

\paragraph{Ablation Study.}
To evaluate the efficacy of \method{}'s components, we conduct an ablation study comparing seven variant configurations, with comprehensive details provided in Appendix~\ref{app:variants} and quantitative results presented in Tables~\ref{tab:ablation}, Table~\ref{tab:ablation_ood1}, and Table~\ref{tab:ablation_ood2} of Appendix~\ref{app:ablation_sup}. The variant \textbf{w/o null-space $\Psi_g$} raises the error and brings violations of constraints. Moving the constraints into the loss \textbf{w/ soft penalty} is worse on both counts. Repairing the field afterwards \textbf{w/ post-projection} does restore both constraints exactly, yet it costs a linear solve at every step and trails \method{}. 
Within the parameterization, the shared boundary potential is what makes the walls exact: removing it (\textbf{w/o boundary $G$}) leaves the divergence at the same numerical floor but lets air through the walls at up to 1.72 times the free-stream speed, at a cost of 19.3--39.7\% in 16-step error. Predicting the updates rather than the absolute field is what keeps the rollout stable. 
Removing any one leads to performance declines, underscoring the critical contributions of these components.

\begin{figure}[t]
\vskip -0.1in
\centering
\includegraphics[width=\textwidth]{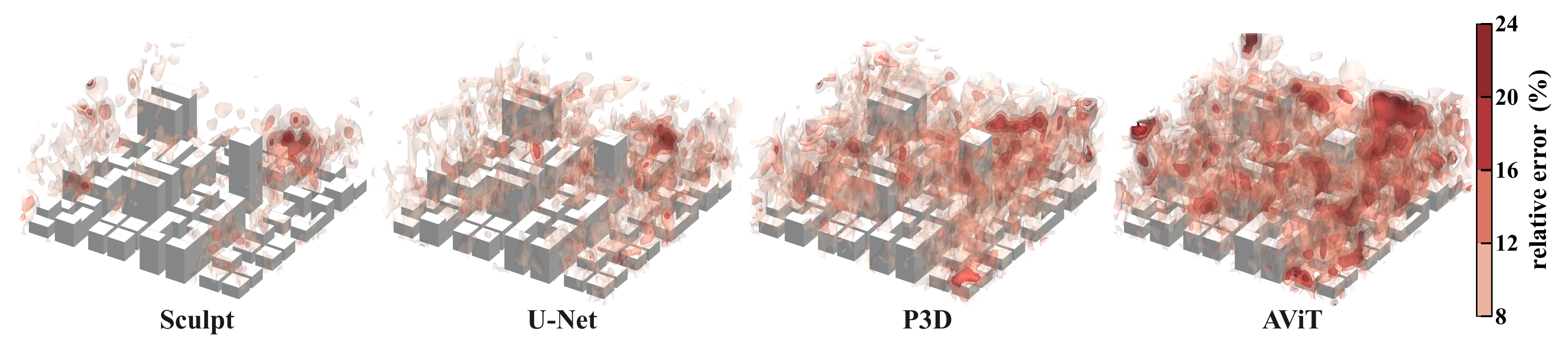}
\caption{Isosurfaces of the prediction error at $k{=}16$ on an unseen inflow case, for \method{} and the three learned baselines. Buildings are grey.}
\label{fig:err3d}
\vskip -0.2in
\end{figure}

\paragraph{Visualizations.}
Figure~\ref{fig:err3d} draws the 16-step rollout error as isosurfaces, which demonstrate that \method{} keeps the error confined to a far smaller volume. Meanwhile, we can find where methods fail. Every one of them is least accurate in the roof-level shear layer: measured against the local wind speed, the mean error of \method{} falls from 9.3\% in the two cells adjacent to a wall to 3.3\% beyond 16\,m from any wall, and the same monotone decay holds for U-Net, P3D and AViT. The region where all four exceed 12\% of the free stream sits at a mean height of 21.5\,m, above the mean roof of 12.9\,m and below the tallest building at 36\,m, which is where the flow separating from the roofs meets the faster air above them. 
One explanation is that shear layer is also the least steady one, which would make it difficult for any surrogate. How to address the problem is left to future work.

%% file: latex/chapters/5_conclusion.tex
\section{Conclusion}
This paper studies how a neural surrogate for transient urban wind can keep its predictions compatible with the geometry. Local mass conservation and wall impermeability constrain shared velocity field, so enforcing one of them disturbs the other. We formulate their joint satisfaction as kinematic admissibility and propose \method{}, which reaches it by a nested volume-boundary potential framework: a volumetric module learns vector potential for fluid cell to get divergence-free, while a surface graph network to learn the scalar potential on the boundary to satisfy the wall impermeability. The two levels assemble into one potential field from which the velocity update is read out, and the guarantee holds inherently.
To measure accuracy and admissibility together we release \bench{}, a building-resolving large-eddy simulation dataset with three evaluation settings. Across them \method{} lowers MSE by 11.7–39.3\% against nine baselines while holding both residuals below $10^{-6}$.

%% file: latex/chapters/6_others.tex
\subsection*{AI use statement}
In this work, the conceptual framework and core ideas presented in this paper represent the authors' original contributions. We used generative AI tools solely for language-related assistance, including manuscript translation and linguistic polishing. We have not used generative AI tools for developing the methodology, designing algorithms, conducting experiments, analyzing results, or generating scientific claims, and other generative AI-related tasks are not applicable to this work. We have reviewed all AI-assisted text and ensured that it accurately reflects the authors' original ideas and contributions. The authors take full responsibility for the final content of this work, including all text, claims, and artifacts produced with the assistance of generative AI tools.

\subsection*{Reproducibility statement}
To facilitate reproducibility, we provide the anonymous code link of the source code and detailed implementation of \method{}. The model configurations, hyper-parameters, training details, and computational resources are provided in Appendix~\ref{app:details}. All datasets used in our experiments are publicly available in our link, with their statistics summarized in Section~\ref{sec:exp}.

%% file: latex/chapters/8_ack.tex

%% file: latex/chapters/7_appendix.tex
\clearpage
\newpage
\appendix

This appendix contains details about proofs, dataset, experimental implementation, supplementary experiments, related works, limitations, and future works.


\section{Proof} \label{proof}

\subsection{Proofs of Propositions~\ref{prop:div} and~\ref{prop:exact}} \label{app:constructions}

\begin{proof}[Proof of Proposition~\ref{prop:div}]
Fix a cell $c$ and sum Eq.~(\ref{eq:curl}) over its six oriented faces. Every edge of $c$ is shared by exactly two of those faces and enters them with opposite incidence signs $\sigma_{f,e}$, so all edge contributions cancel in pairs and $(\nabla_h\!\cdot\!\Delta\mathbf{u})_c=0$. The argument uses only the incidence structure of the grid, hence holds for every $\mathbf{A}$ and every mask $g$.
\end{proof}

\begin{proof}[Proof of Proposition~\ref{prop:exact}]
Every $f\in\mathcal{F}_{\partial}(g)$ has all four of its edges in $\mathcal{E}_{\partial}(g)$, so Eq.~(\ref{eq:curl}) evaluates $\Delta u_f$ from boundary values alone. The first term of Eq.~(\ref{eq:hodge}) is a discrete gradient, whose circulation around the oriented boundary of $f$ telescopes to zero. The second term is a closed $1$-cochain on $\mathcal{B}$, so its circulation around the $2$-cell $f$ vanishes by definition of closedness. Summing the two gives $\Delta u_f=0$.
\end{proof}

\subsection{Spectral Gain of the Curl Readout} \label{app:adjoint}

This appendix proves the statement made in Section~\ref{sec:volume}, that eliminating the constraints through a potential attenuates the gradient signal that reaches the potential at large scales and amplifies it at small ones. 

{\bfseries Setting and Notation.}
Throughout this appendix we work on the uniform grid of spacing $h$ with periodic boundary conditions in all three directions, which isolates the effect of the readout from that of the boundary treatment of Section~\ref{sec:rep}. Write $\mathcal{R}$ for the linear map that assembles the update from the volume vector potential, so that $\Delta\mathbf{u}=\mathcal{R}\mathbf{A}$ is Eq.~(\ref{eq:curl}) written for all faces at once, and write $\mathcal{R}^{\!\top}$ for its adjoint. For any loss $\mathcal{L}$ the chain rule gives
\begin{equation} \label{eq:adjoint}
    \frac{\partial \mathcal{L}}{\partial \mathbf{A}}
    \;=\; \mathcal{R}^{\!\top}\,\mathbf{r},
    \qquad
    \mathbf{r} \;=\; \frac{\partial \mathcal{L}}{\partial \Delta\mathbf{u}} ,
\end{equation}
so the quantity to be characterized is the gain of $\mathcal{R}^{\!\top}$ as a function of spatial scale. Let $\boldsymbol{\xi}=(\xi_1,\xi_2,\xi_3)$ denote an admissible wavevector of the periodic grid, with $\xi_\alpha\in(-\pi/h,\pi/h]$ for $\alpha\in\{1,2,3\}$, and write $\mathbf{A}_{\boldsymbol{\xi}}$ and $\mathbf{r}_{\boldsymbol{\xi}}$ for the discrete Fourier coefficients of $\mathbf{A}$ and of $\mathbf{r}$ at $\boldsymbol{\xi}$. On a uniform periodic grid, $\mathcal{R}$ is translation-invariant and hence is block-diagonalized by the discrete Fourier transform, with one $3\!\times\!3$ block $\mathcal{R}_{\boldsymbol{\xi}}$ per wavevector. Write $D_\alpha$ for the one-sided difference along axis $\alpha$ and define
\begin{equation} \label{eq:symbol}
    q_\alpha(\boldsymbol{\xi}) \;=\; \frac{e^{\mathrm{i}\xi_\alpha h}-1}{h},
    \qquad
    |q_\alpha(\boldsymbol{\xi})| \;=\; \frac{2}{h}\Big|\sin\frac{\xi_\alpha h}{2}\Big| ,
\end{equation}
so that $q_\alpha$ is the symbol of $D_\alpha$. We collect the three symbols into the vector $\mathbf{q}(\boldsymbol{\xi})$.

\begin{proposition} \label{prop:gain}
The block $\mathcal{R}_{\boldsymbol{\xi}}$ is the cross-product matrix of $\mathbf{q}(\boldsymbol{\xi})$,
and for every $\boldsymbol{\xi}\neq\mathbf{0}$,
\begin{equation} \label{eq:gram}
    \mathcal{R}_{\boldsymbol{\xi}}^{*}\mathcal{R}_{\boldsymbol{\xi}} \;=\; |\mathbf{q}(\boldsymbol{\xi})|^{2}\,
    \big(\mathbf{I}-\Pi_{\boldsymbol{\xi}}\big),
    \qquad
    \Pi_{\boldsymbol{\xi}} \;=\; \frac{\mathbf{q}(\boldsymbol{\xi})\,\mathbf{q}(\boldsymbol{\xi})^{*}}{|\mathbf{q}(\boldsymbol{\xi})|^{2}} ,
\end{equation}
where $\Pi_{\boldsymbol{\xi}}$ is the orthogonal projector onto the span of $\mathbf{q}(\boldsymbol{\xi})$, which is the range of the discrete gradient at that wavevector. Consequently the singular values of $\mathcal{R}_{\boldsymbol{\xi}}$ are $|\mathbf{q}(\boldsymbol{\xi})|$ twice and $0$ once, and the same holds for $\mathcal{R}_{\boldsymbol{\xi}}^{*}$.
\end{proposition}

\begin{proof}
On the staggered grid the three components of Eq.~(\ref{eq:curl}) read $\Delta u = D_2 A^{3}-D_3 A^{2}$, $\Delta v = D_3 A^{1}-D_1 A^{3}$ and $\Delta w = D_1 A^{2}-D_2 A^{1}$, where $A^{\alpha}$ collects the entries of $\mathbf{A}$ on the edges aligned with axis $\alpha$. Applying the discrete Fourier transform turns $D_\alpha$ into multiplication by $q_\alpha(\boldsymbol{\xi})$ of Eq.~(\ref{eq:symbol}), so $\mathcal{R}_{\boldsymbol{\xi}}\mathbf{A}_{\boldsymbol{\xi}} =\mathbf{q}(\boldsymbol{\xi})\times\mathbf{A}_{\boldsymbol{\xi}}$, which is the first claim. Cross-product matrices are antisymmetric, so $\mathcal{R}_{\boldsymbol{\xi}}^{*}=-[\bar{\mathbf{q}}]_{\times}$, and the bilinear identity $[\mathbf{a}]_{\times}[\mathbf{b}]_{\times}=\mathbf{b}\mathbf{a}^{\top}-(\mathbf{a}^{\top}\mathbf{b})\mathbf{I}$ gives
\begin{equation}
    \mathcal{R}_{\boldsymbol{\xi}}^{*}\mathcal{R}_{\boldsymbol{\xi}} \;=\; -[\bar{\mathbf{q}}]_{\times}[\mathbf{q}]_{\times}
    \;=\; (\bar{\mathbf{q}}^{\top}\mathbf{q})\,\mathbf{I}-\mathbf{q}\,\bar{\mathbf{q}}^{\top}
    \;=\; |\mathbf{q}|^{2}\mathbf{I}-\mathbf{q}\mathbf{q}^{*} ,
\end{equation}
which is Eq.~(\ref{eq:gram}). The matrix $\mathbf{q}\mathbf{q}^{*}/|\mathbf{q}|^{2}$ is Hermitian, idempotent and of rank one, hence the orthogonal projector onto $\operatorname{span}(\mathbf{q})$, and $\mathbf{q}$ is exactly the Fourier block of the discrete gradient, whose entries are the same one-sided differences. Eq.~(\ref{eq:gram}) therefore has eigenvalues $|\mathbf{q}|^{2}$ with multiplicity two and $0$ with multiplicity one, which are the squared singular values of $\mathcal{R}_{\boldsymbol{\xi}}$. At $\boldsymbol{\xi}=\mathbf{0}$ one has $\mathbf{q}=\mathbf{0}$ and $\mathcal{R}_{\mathbf{0}}=0$, so a spatially constant potential produces no update.
\end{proof}

The projector in Eq.~(\ref{eq:gram}) is the Fourier statement of the identity that a discrete gradient has no curl, which Section~\ref{sec:rep} uses at the boundary: the one direction that $\mathcal{R}_{\boldsymbol{\xi}}$ annihilates is the gradient direction.

\begin{corollary} \label{cor:gain}
Let $\boldsymbol{\xi}\neq\mathbf{0}$ and let $\mathbf{r}^{\parallel}_{\boldsymbol{\xi}}$ be the component of $\mathbf{r}_{\boldsymbol{\xi}}$ in the range of $\mathcal{R}_{\boldsymbol{\xi}}$. Then the gradient that reaches the potential satisfies
\begin{equation} \label{eq:gainbound}
    \big\|\mathcal{R}_{\boldsymbol{\xi}}^{*}\,\mathbf{r}_{\boldsymbol{\xi}}\big\|
    \;=\; |\mathbf{q}(\boldsymbol{\xi})|\;\big\|\mathbf{r}^{\parallel}_{\boldsymbol{\xi}}\big\|
    \;\leq\; |\mathbf{q}(\boldsymbol{\xi})|\;\big\|\mathbf{r}_{\boldsymbol{\xi}}\big\| ,
\end{equation}
with
\begin{equation} \label{eq:gainlimits}
    |\mathbf{q}(\boldsymbol{\xi})|^{2}=\frac{4}{h^{2}}\sum_{\alpha=1}^{3}\sin^{2}\frac{\xi_\alpha h}{2},
    \qquad
    |\mathbf{q}(\boldsymbol{\xi})| = |\boldsymbol{\xi}| + O(|\boldsymbol{\xi}|^{3}h^{2}) \ \text{ as } |\boldsymbol{\xi}|h\to 0,
    \qquad
    \max_{\boldsymbol{\xi}}|\mathbf{q}(\boldsymbol{\xi})| = \frac{2\sqrt{3}}{h} .
\end{equation}
The gain $|\mathbf{q}(\boldsymbol{\xi})|$ is nondecreasing in each $|\xi_\alpha|$ on $[0,\pi/h]$.
\end{corollary}

\begin{proof}
Eq.~(\ref{eq:gainbound}) is Proposition~\ref{prop:gain} applied to $\|\mathcal{R}_{\boldsymbol{\xi}}^{*}\mathbf{r}_{\boldsymbol{\xi}}\|^{2}
=\mathbf{r}_{\boldsymbol{\xi}}^{*}\mathcal{R}_{\boldsymbol{\xi}}\mathcal{R}_{\boldsymbol{\xi}}^{*}\mathbf{r}_{\boldsymbol{\xi}}$, together with $\mathcal{R}_{\boldsymbol{\xi}}\mathcal{R}_{\boldsymbol{\xi}}^{*}=|\mathbf{q}|^{2}(\mathbf{I}-\Pi^{\perp}_{\boldsymbol{\xi}})$, where $\Pi^{\perp}_{\boldsymbol{\xi}}$ projects onto $\ker\mathcal{R}_{\boldsymbol{\xi}}^{*}$, which is the orthogonal complement of the range of $\mathcal{R}_{\boldsymbol{\xi}}$. The remaining statements follow from Eq.~(\ref{eq:symbol}) and from $\sin x = x + O(x^{3})$, with the maximum attained when every $|\xi_\alpha|$ equals $\pi/h$.
\end{proof}

The factor by which the gradient is rescaled on its way to the potential is $|\mathbf{q}(\boldsymbol{\xi})|$, which vanishes linearly as the wavelength grows and is largest at the grid scale, so a mode of wavelength $2\pi/|\boldsymbol{\xi}|$ receives a gradient smaller by a factor of order $|\boldsymbol{\xi}|h$ than a mode at the grid scale receives, and the gain at the largest resolved wavelength falls below the gain at the Nyquist wavenumber by a factor of order the ratio of the domain size to the grid spacing. Any elimination of a differential constraint through a potential contributes the singular values of the eliminating operator in the same way, so the attenuation is a property of the change of variables. The component of the face-side gradient outside the range of $\mathcal{R}_{\boldsymbol{\xi}}$ never reaches the potential, which is the gradient-side counterpart of the fact that the construction emits only fields inside $\mathcal{V}(g)$.

\begin{remark} \label{rem:gaintwice}
On a flat portion of the boundary the same computation applies to the composition through which the shared scalar potential acts, namely a gradient followed by a curl. Each factor contributes one power of $|\mathbf{q}(\boldsymbol{\xi})|$, so the gain from the loss to $G$ carries $|\mathbf{q}(\boldsymbol{\xi})|^{2}$ and the attenuation at large scales is squared. This scaling motivates the multi-resolution readout of \(G\) in Section~\ref{sec:surface}. The boundary complex is not a periodic lattice, so the argument is exact where the boundary is flat and serves as a guide elsewhere.
\end{remark}

\section{Data Details} \label{app:data}

\subsection{Dataset Generation} \label{app:dataset}

Every layout is a building height field on the horizontal grid, extruded from the ground so that each building is a filled column, and is fixed by a morphology type and a random seed. The type sets the composition of the sixteen blocks that tile the domain, from archetypes differing in footprint, spacing and storey count: courtyard blocks, rows of elongated buildings, isolated towers, towers on a low podium, low-rise fabric and open ground. It thereby fixes the plan-area density and the height distribution. The seed draws the street-grid offsets, the archetype of each block, the storey count of each building and the placement of each footprint. Streets run straight through the domain at widths of 24, 16 and 6\,m on a block rhythm of 128\,m, which divides the domain width and so respects the horizontal periodicity. Storeys are 3\,m high and every height is quantized to $h$, so that no building edge cuts a cell.

All sequences come from the PALM large-eddy simulation model~\citep{maronga2020overview} on a $24\!\times\!128\!\times\!128$ grid of uniform spacing $h=2$\,m, a domain of $48\!\times\!256\!\times\!256$\,m that is doubly periodic horizontally, bounded below by rough ground and above by a free-slip lid. The solver blocks the cells inside a building and writes velocities on the native Arakawa-C staggered grid. Subgrid stresses use the $1.5$-order closure with the mixing length limited near surfaces, advection uses the Piacsek--Williams scheme, and time stepping is third-order Runge--Kutta with the perturbation pressure solved at every substep. That projection uses the multigrid solver with masking disabled, an unbounded cycle count, a residual limit of $10^{-13}$ and four Gauss--Seidel sweeps per level. The roughness length is $z_0=0.1$\,m on the ground and on all building surfaces, which carry a stress-type condition from Monin-Obukhov similarity theory~\citep{monin1954basic}. The normal velocity vanishes on the ground, the buildings and the lid, and the Coriolis parameter is zero.

Each run starts from constant profiles at a surface potential temperature of $300$\,K with small random disturbances and integrates for $1{,}000$\,s, after which $100$ consecutive snapshots at intervals of $1$\,s form one sequence. The inflow enters through the two surface geostrophic wind components, which carry the free-stream speed and the angle to the fixed street grid, so that direction varies independently of geometry. Stratification enters through the surface heat flux and the temperature lapse rate: $0.1$\,K\,m\,s$^{-1}$ with no lapse for the convective condition, zero with no lapse for the neutral one, and $-0.02$\,K\,m\,s$^{-1}$ with $0.01$\,K\,m$^{-1}$ for the stable one. The nine inflow conditions combine three wind directions ($15^{\circ}$, $30^{\circ}$, $45^{\circ}$), three free-stream speeds ($3$, $6$, $9$\,m\,s$^{-1}$), and stability (convective, neutral, stable). The three conditions held out for testing are $(15^{\circ}, 9\,\mathrm{m\,s^{-1}}, \mathrm{convective})$, $(30^{\circ}, 3\,\mathrm{m\,s^{-1}}, \mathrm{neutral})$ and $(45^{\circ}, 6\,\mathrm{m\,s^{-1}}, \mathrm{stable})$. \bench{} holds $85$ such sequences. A PALM run uses 16 MPI processes.
During the whole experiments, we use an Intel(R) Xeon(R) Gold 5318Y CPU and one NVIDIA A800 80GB PCIe GPU with 80GB memory to train and inference.

\begin{table*}[h]
\centering
\small
\caption{Related flow-field benchmark that wind surrogates could be evaluated on. \textbf{Solver} names the method that generated the fields. \textbf{3D} marks volumetric fields. \textbf{Obstacles} marks solid bodies that sit inside the domain and are resolved by the grid. \textbf{Transient} marks instantaneous slices. \textbf{Geometry} marks the geometric shift.}
\label{tab:datasets}
\resizebox{\textwidth}{!}{
\begin{tabular}{llcccc}
\toprule
\textbf{Dataset} & \textbf{Solver} & \textbf{3D} & \textbf{Obstacles} & \textbf{Transient} & \textbf{Geometry} \\
\midrule
PDEBench~\citep{takamoto2022pdebench}         & Canonical PDE solvers   & $\checkmark$ & $\times$     & $\checkmark$ & $\times$ \\
The Well~\citep{ohana2024well}         & Multi-physics solvers   & $\checkmark$ & $\times$     & $\checkmark$ & $\times$ \\
PDEArena~\citep{gupta2022towards}         & Canonical PDE solvers   & $\times$     & $\times$     & $\checkmark$ & $\times$ \\
JHTDB~\citep{li2008public}            & DNS                     & $\checkmark$ & $\times$     & $\checkmark$ & $\times$ \\
Gen4Turbulence~\citep{oommen2025learning}   & DNS                     & $\checkmark$ & $\times$     & $\checkmark$ & $\times$ \\
Geometry Matters~\citep{rabeh2024geometry} & Steady 2D solver        & $\times$     & $\checkmark$ & $\times$     & $\checkmark$ \\
UrbanFlow-3K~\citep{lee2026urbanflow}     & LBM, 2D                 & $\times$     & $\checkmark$ & $\times$     & $\checkmark$ \\
UrbanTALES~\citep{nazarian2025urbantales}       & LES (PALM)              & $\checkmark$ & $\checkmark$ & $\times$     & $\checkmark$ \\
U3DWind~\citep{zhou2026u3dwind}          & LBM-LES                 & $\checkmark$ & $\checkmark$ & $\times$     & $\checkmark$ \\
\midrule
\textbf{\bench{} (ours)} & \textbf{LES (PALM)} & $\checkmark$ & $\checkmark$ & $\checkmark$ & $\checkmark$ \\
\bottomrule
\end{tabular}
}
\end{table*}

\subsection{Comparison with Existing Benchmarks} \label{app:compare}
Measuring Eq.~(\ref{eq:pre-admissible}) requires data in which the solid boundaries are present and resolved, and measuring it under time stepping requires instantaneous snapshots. Table~\ref{tab:datasets} summarizes the public benchmarks against the two requirements.
To our knowledge no publicly released dataset provides volumetric, instantaneous velocity fields over resolved urban geometry. \bench{} is built to fill this gap. We use the tensor layout, resolution class and snapshot count of Gen4Turbulence~\citep{oommen2025learning} as a reference, and extend it with the two elements that our setting requires, a per-layout solid-occupancy mask and a sample axis spanning procedurally generated building layouts.


\section{Experimental Details} \label{app:exp}

\subsection{Implementation Details.} \label{app:details}

During the evaluation, all velocities are normalized by the free-stream speed, and the divergence is computed from these normalized velocities. The wall residual is taken over every face that is not bounded by fluid on both sides. Within every training sequence the frames are split as [3,70) for training and [76,84) for validation, with a six-frame gap so that no validation frame is temporally adjacent to a training frame. Training and evaluation are separated over sequences: a test sequence is absent from training at every frame, so the two may occupy the same time range. Every test sequence contributes the same 18 start frames $t_0=66,\ldots,83$, each rolled out to $t_0+16\le 99$.

All evaluation use $n=4$ history frames. Every method receives the same input and the same split, and is trained for 120 epochs with Adam at a learning rate of $10^{-3}$, no learning-rate schedule. \method{} is trained with a mean squared reconstruction term alone, with no penalty on the constraints, and it is trained with the pushforward scheme of~\citet{brandstetter2022message}: with probability $p$ the most recent input frame is replaced by the model's own one-step prediction, taken without gradient, and the target is advanced by one step so that it remains the increment relative to that prediction.
All convolutions pad circularly along the two horizontal axes, which is what the doubly periodic domain requires, and with zeros along the vertical, where the ground and the lid are physical boundaries. This applies to \method{} and to every convolutional baseline alike.
\method{} uses backbone widths $C_0,C_1,C_2=32,64,128$ over $S=3$ resolution levels, multiresolution width $c_{\mathcal{A}}=16$, surface width $c_{\mathcal{B}}=32$, descriptor width $d_{\mathrm{des}}=8$ and $K=4$ message-passing jumps.
For every method, including \method{}, we keep the epoch at which the 16-step rollout error on the validation frames is lowest, following the rollout-based selection of PDEArena~\citep{gupta2022towards}.

\subsection{Ablation Variants} \label{app:variants}

The first four variants concern where the constraint is imposed.

{\bfseries w/o null-space $\Psi_g$.}
The structured output representation of Eq.~(\ref{eq:pre-nullspace}) is removed and the network regresses the three velocity components directly, through the same read-out head that the baselines use.

{\bfseries w/ soft penalty.}
Starting from \textsc{w/o null-space $\Psi_g$}, the training objective gains a divergence term and a wall term, $\lambda(\|\nabla_h\!\cdot\!\mathbf{u}\|^2+\|u_f\|^2)$ with $\lambda=1$ on both. 

{\bfseries w/ post-projection.}
Starting from \textsc{w/o null-space $\Psi_g$}, and without any further training, each predicted step is passed through a pressure projection: the normal velocity is cleared on solid faces, a Poisson problem $\nabla_h^2\phi=\nabla_h\!\cdot\!\mathbf{u}$ with homogeneous Neumann conditions on solid boundaries is solved by conjugate gradients in double precision to a tolerance of $10^{-10}$, and the gradient of $\phi$ is subtracted. 

{\bfseries w/o boundary $G$.}
The discrete curl of Eq.~(\ref{eq:curl}) is kept, so Proposition~\ref{prop:div} still applies, but the substitution of Eq.~(\ref{eq:hodge}) on boundary edges is skipped: those edges retain the free output of the network instead of being set to a discrete gradient of the boundary potential. 

The remaining variants vary design choices that are not required for either guarantee.

{\bfseries w/o increment $\Delta\mathbf{u}$.}
The network predicts the absolute field $\mathbf{u}^{t+1}$ instead of the increment $\mathbf{u}^{t+1}-\mathbf{u}^{t}$.

{\bfseries w/o multi-scale $\mathbf{A}$.}
The set of resolutions in Eq.~(\ref{eq:multires}) is reduced from $\mathcal{S}=\{1,2,4\}$ to $\mathcal{S}=\{1\}$, so the potential is read out from the finest backbone features alone.

{\bfseries w/o learned fusion.}
At each boundary node the surface network gathers the eight surrounding cells of the backbone feature map and combines them with weights produced by a small learned scorer. Here that scorer is deleted and the eight slots are averaged uniformly over the valid ones.




\section{Supplementary experiments} \label{app:results}

\subsection{Performance on \textbf{unseen inflows} setting} \label{app:main_sup}

Table~\ref{tab:ood2} reports the \textbf{unseen inflows} setting, whose test sequences hold out three combinations of wind direction, free-stream speed and stability. \method{} is first at all four horizons, ahead of the strongest neural baseline by $55.7\%$ at $k{=}1$ and by $33.2\%$ at $k{=}16$. This setting separates the methods differently from the other two: every neural baseline degrades far enough by $k{=}16$ that POD overtakes all of them, and AFNO leaves the range of the table altogether. 

\begin{table*}[t]
    \centering
    \caption{Performance on the \textbf{unseen inflows} setting of \bench{}. MSE is reported at four rollout horizons. Highlighted are the top \textcolor{mydeepred}{first}, \textcolor{mydeepblue}{second} results. `$<$' in two constraint columns means the power-of-ten bound on the measured value, and $0$ is exact.}
    \label{tab:ood2}
    \renewcommand{\arraystretch}{1.2}
    \resizebox{\linewidth}{!}{
    \begin{tabular}{l|cccc|cc}
        \toprule
        \multirow{2}{*}{Methods} & \multicolumn{4}{c|}{MSE$\downarrow$ ($\times 10^{-3}$)} & \multicolumn{2}{c}{worst over all $k$} \\\cmidrule{2-5}\cmidrule{6-7}
        & \multicolumn{1}{c}{$k{=}1$} & \multicolumn{1}{c}{$k{=}4$} & \multicolumn{1}{c}{$k{=}8$} & \multicolumn{1}{c|}{$k{=}16$} & \multicolumn{1}{c}{$\max|\nabla\!\cdot\!\mathbf{u}|\downarrow$} & \multicolumn{1}{c}{$\max|\mathbf{u}\!\cdot\!\mathbf{n}|\downarrow$} \\
        \midrule
        DMD & 87.65\small{$\pm$27.47} & 87.49\small{$\pm$27.13} & 98.42\small{$\pm$29.17} & 100.31\small{$\pm$26.95} & $<10^{-6}$ & $0$ \\
        Time-Mean & 47.47\small{$\pm$15.62} & 73.29\small{$\pm$22.01} & 80.51\small{$\pm$24.51} & 82.36\small{$\pm$23.10} & $<10^{-6}$ & $0$ \\
        POD & 50.64\small{$\pm$15.09} & 51.21\small{$\pm$14.96} & 52.01\small{$\pm$15.16} & \textcolor{mydeepblue}{52.44\small{$\pm$15.12}} & $<10^{-5}$ & 0.30 \\
        Dil-ResNet & 21.56\small{$\pm$10.67} & 56.34\small{$\pm$17.34} & 84.91\small{$\pm$21.44} & 201.57\small{$\pm$42.63} & 2.09 & 2.26 \\
        TFNO & 35.86\small{$\pm$10.78} & 82.91\small{$\pm$23.58} & 137.90\small{$\pm$50.52} & 188.84\small{$\pm$86.16} & 0.83 & 1.60 \\
        AFNO & 22.19\small{$\pm$8.15} & 51.87\small{$\pm$15.38} & 64.85\small{$\pm$17.64} & $1.6{\times}10^{3}$ & $5.9{\times}10^{2}$ & 2.63 \\
        AViT & 14.64\small{$\pm$4.59} & 42.54\small{$\pm$13.20} & 70.31\small{$\pm$18.66} & 93.57\small{$\pm$22.99} & 1.40 & 1.64 \\
        P3D & 3.02\small{$\pm$1.90} & 19.57\small{$\pm$14.73} & 44.39\small{$\pm$24.34} & 83.77\small{$\pm$26.10} & 0.37 & 1.70 \\
        U-Net & \textcolor{mydeepblue}{2.01\small{$\pm$1.29}} & \textcolor{mydeepblue}{10.85\small{$\pm$5.60}} & \textcolor{mydeepblue}{30.62\small{$\pm$12.51}} & 69.36\small{$\pm$17.57} & 0.24 & 1.30 \\
        \hline\hline
        \textbf{\method{}} & \textcolor{mydeepred}{\textbf{0.89\small{$\pm$0.67}}} & \textcolor{mydeepred}{\textbf{6.35\small{$\pm$5.41}}} & \textcolor{mydeepred}{\textbf{19.53\small{$\pm$15.00}}} & \textcolor{mydeepred}{\textbf{46.31\small{$\pm$25.52}}} & $<10^{-6}$ & $<10^{-6}$ \\
        \bottomrule
    \end{tabular}
    }
    \vskip -0.1in
\end{table*}

\subsection{Ablation on \textbf{unseen inflows} setting} \label{app:ablation_sup}

Table~\ref{tab:ablation_ood2} repeats the ablation of Section~\ref{main} on the same setting, where \method{} is again first at both horizons. Two components change rank here. The multiresolution potential and its learned fusion carry far more weight under inflow shift, costing $30.7\%$ and $51.6\%$ at $k{=}16$ against $1.8\%$ and $16.0\%$ on the unseen layout setting.

\begin{table*}[h]
    \centering
    \caption{Ablation on the \textbf{unseen inflows} setting of \bench{}. MSE is reported at two rollout horizons. Highlighted are the top \textcolor{mydeepred}{first}, \textcolor{mydeepblue}{second} results. `$<$' in two constraint columns means the power-of-ten bound on the measured value, and $0$ is exact.}
    \label{tab:ablation_ood2}
    \renewcommand{\arraystretch}{1.15}
    \resizebox{\textwidth}{!}{
    \begin{tabular}{l|cc|cc|cc}
        \toprule
        \multirow{2}{*}{Variant} & \multicolumn{2}{c|}{MSE$\downarrow$} & \multicolumn{2}{c|}{$\max|\nabla\!\cdot\!\mathbf{u}|\downarrow$} & \multicolumn{2}{c}{$\max|\mathbf{u}\!\cdot\!\mathbf{n}|\downarrow$} \\\cmidrule{2-7}
        & \multicolumn{1}{c}{$k{=}1$} & \multicolumn{1}{c|}{$k{=}16$} & \multicolumn{1}{c}{$k{=}1$} & \multicolumn{1}{c|}{$k{=}16$} & \multicolumn{1}{c}{$k{=}1$} & \multicolumn{1}{c}{$k{=}16$} \\
        \midrule
        \textbf{\method{}} & \textcolor{mydeepred}{\textbf{0.89\small{$\pm$0.67}}} & \textcolor{mydeepred}{\textbf{46.31\small{$\pm$25.52}}} & $<10^{-6}$ & $<10^{-6}$ & $<10^{-6}$ & $<10^{-6}$ \\
        \hline\hline
        w/o null-space $\Psi_g$ & 1.71\small{$\pm$1.52} & 63.81\small{$\pm$32.34} & 0.28 & 1.24 & 1.37 & 3.32 \\
        w/ soft penalty & 12.89\small{$\pm$3.36} & 109.27\small{$\pm$23.41} & 0.70 & 0.84 & 0.51 & 0.71 \\
        w/ post-projection & 1.63\small{$\pm$1.49} & 60.98\small{$\pm$33.39} & $<10^{-6}$ & $<10^{-6}$ & $0$ & $0$ \\
        w/o boundary $G$ & 1.93\small{$\pm$2.07} & 64.68\small{$\pm$38.22} & $<10^{-6}$ & $<10^{-6}$ & 1.34 & 1.72 \\
        w/o increment $\Delta\mathbf{u}$ & 3.18\small{$\pm$2.13} & 2122.96\small{$\pm$178.64} & $<10^{-6}$ & $<10^{-5}$ & $<10^{-6}$ & $<10^{-5}$ \\
        w/o multi-scale $\mathbf{A}$ & 2.56\small{$\pm$3.50} & \textcolor{mydeepblue}{60.51\small{$\pm$31.26}} & $<10^{-6}$ & $<10^{-6}$ & $<10^{-6}$ & $<10^{-6}$ \\
        w/o learned fusion & \textcolor{mydeepblue}{1.05\small{$\pm$0.78}} & 70.22\small{$\pm$37.12} & $<10^{-6}$ & $<10^{-6}$ & $<10^{-6}$ & $<10^{-6}$ \\
        \bottomrule
    \end{tabular}
    }
\end{table*}


\subsection{Computational complexity} \label{app:cost}
We analyze the computational complexity of \method{} to show its efficiency. Let $N=D\cdot H\cdot W$ denote the number of grid cells, $|\mathcal{N}|$ and $|\mathcal{E}_{\partial}(g)|$ represent the numbers of nodes and of edges of the boundary complex, $C_0$ indicate the width of the finest backbone level, and $c_{\mathcal{A}}$, $c_{\mathcal{B}}$ and $d_{\mathrm{des}}$ the widths of the multiresolution branches, of the surface graph network and of the node descriptor, respectively, with $S$ resolution levels and $K$ rounds of message passing. A $3\!\times\!3\!\times\!3$ convolution at level $s$ acts on $N/8^{s}$ cells with $C_s=2^{s}C_0$ channels, so the computational complexity of the model during the volumetric stage is given by $O\big(\sum_{s}(C_s^{2}+c_{\mathcal{A}}C_s)\,N/8^{s}+SN\big)$. Generally speaking, $C_0$, $c_{\mathcal{A}}$ and $S$ are significantly smaller than $N$, and the geometric sum over the levels is dominated by the finest one. So the time complexity of the first stage of \method{} is $O(C_0^{2}N)$.

The computational complexity of the second stage is $O\big(|\mathcal{N}|\,c_{\mathcal{B}} (C_0+d_{\mathrm{des}}+Kc_{\mathcal{B}}+S)\big)$, where the first term collects the learned gather of Eq.~(\ref{eq:gather}) over the eight cells incident to each node, the second collects the $K$ rounds of Eq.~(\ref{eq:gnn}) with one $c_{\mathcal{B}}\!\times\!c_{\mathcal{B}}$ matrix per axis direction, and the third collects the read-out of Eq.~(\ref{eq:gread}). The assembly operator writes one value on each edge of $\mathcal{E}_{\partial}(g)$ and the discrete curl of Eq.~(\ref{eq:curl}) sums four terms on each face, which together cost $O(|\mathcal{E}_{\partial}(g)|+N)$ and carry no parameters and no linear system. On \bench{} the boundary complex spans between $9.6\%$ and $12.3\%$ of the cells, so $|\mathcal{N}|$ and $|\mathcal{E}_{\partial}(g)|$ are $O(N)$ with a small constant. Therefore, the time complexity of \method{} is $O(C_0^{2}N+Kc_{\mathcal{B}}^{2}|\mathcal{N}|+N)$, which reduces to $O(N)$ at fixed widths. Although our model enforces both constraints by construction, it has comparable complexity with the convolutional baselines, whose cost is $O(C^{2}N)$ at the same resolution.

\subsection{Parameters Analysis} \label{app:params}

\begin{figure}[h]
\vskip -0.1in
    \centering
    \includegraphics[width=\textwidth]{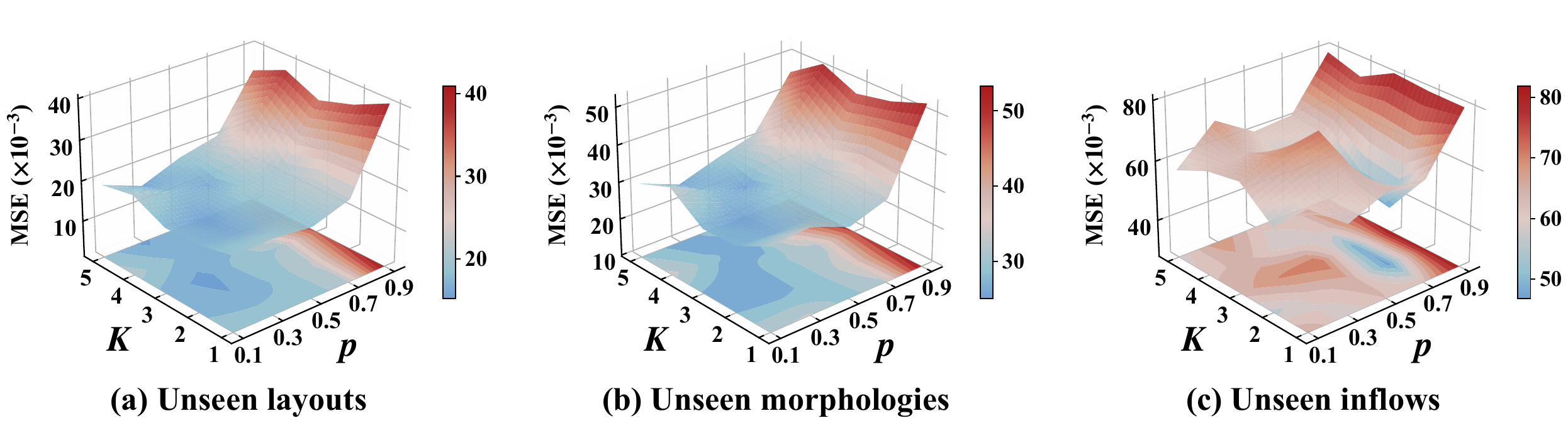}
    \caption{Sensitivity of the $16$-step MSE to the number of message-passing rounds $K$ on the boundary complex and to the pushforward probability $p$.}
    \label{fig:hpgrid}
\end{figure}

\method{} introduces two hyper-parameters: the number of message-passing jump $K$ on the boundary complex and the pushforward probability $p$. Figure~\ref{fig:hpgrid} reports the $16$-step MSE of different hyper-parameters setting.
The shape along $p$ is the same in all three settings: a plateau up to $p=0.5$, followed by an increase error at $p=0.9$. Training almost exclusively on the model's own predictions leaves few steps anchored to the ground truth, leading to the instability of training process. The effect of the $K$ value is relatively stable, and the optimal value is around 4. We therefore fix $K=4$ and $p=0.5$ for all evaluation settings.

\subsection{Visualization} \label{app:vis}

Figures~\ref{fig:err3d_campus}, \ref{fig:err3d_park} and~\ref{fig:err3d_plaza} extend Figure~\ref{fig:err3d} to three further test sequences and to all four reported horizons. Each row of a figure shows one horizon, and the four panels of a row place \method{} beside U-Net, P3D and AViT on the same field with one common colour scale, so that panels are comparable within a row and across rows.
Three observations hold on all three sequences. The error volume grows with the horizon for every method, and the shells of the baselines merge into connected sheets between $k=8$ and $k=16$ while those of \method{} remain separated. At $k=1$ the four methods are close, and the ordering that the tables report at $k=16$ is already visible at $k=4$. The regions that fill first lie in the layer above the roofs, which is where Section~\ref{main} locates the largest error of every method.


\section{Related Works}

\paragraph{Neural surrogates for urban wind.}
Surrogates for neighbourhood-scale wind divide into two families according to what they return. The first predicts a field in a single forward pass. U-Net architectures regress the pedestrian-level field from a layout mask and recover its small scales with a generative refiner~\citep{hoeiness2021positional,kastner2023gan,snaiki2026hierarchical,wu2026multi,wang2025evaluating,du2026st}. Neural operators transfer across wind directions and across cities~\citep{chen2026urban,peng2024fourier,qin2025flume,qin2025modeling}, and geometry-aware operators evaluate on arbitrary domains~\citep{wu2024transolver,wen2026geometry,de2026anchored,wang2025geometry}. The fields these methods return are steady or time-averaged, and in most cases two-dimensional. The second family consists of general neural PDE surrogates~\citep{gupta2022towards,brandstetter2022message,kochkov2021machine}, which learn a single-step evolution operator with convolutional, spectral or attention mechanisms~\citep{ronneberger2015u,stachenfeld2021learned,li2020fourier,guibas2021adaptive,holzschuh2025p3d} and apply it to their own output over a rollout~\citep{lippe2023pde,mccabe2023towards,kohl2026benchmarking}. These methods return three-dimensional transient fields, and the layout enters them as an input channel that conditions the fitted map. \method{} belongs to the second family, and the layout additionally determines the space in which its output is expressed.

\paragraph{Constraint satisfaction in neural surrogates.}
Three mechanisms are in use. Penalty methods add the residual of the constraint to the training objective~\citep{raissi2019physics,karniadakis2021physics}, so that the constraint holds to the accuracy the optimiser reaches. Projection methods correct the prediction after inference, typically through a pressure solve, and restore the constraint exactly at the cost of a linear solve at every step. Structured output representations satisfy the constraint by construction, through divergence-free parameterisations and vector potentials~\citep{beucler2021enforcing,wandel2021teaching,richter2022neural,mohan2023embedding,huang2025geometric,li2026physics} or through geometry-dependent compatible spaces for steady problems~\citep{shaffer2026structure}. The variants \textsc{w/ soft penalty} and \textsc{w/ post-projection} of Appendix~\ref{app:variants} instantiate the first two mechanisms on the backbone of \method{}. Existing structured representations are formulated on periodic domains, or on domains whose boundary is aligned with the grid. A building surface is immersed in a structured grid~\citep{mittal2005immersed} and moves with the layout, so the two constraints considered here act on a shared set of degrees of freedom.

\paragraph{Spectral properties of learned evolution operators.}
The tendency of neural networks to fit low frequencies first is documented~\citep{rahaman2019spectral,xu2019frequency, oommen2025learning} and follows from the optimisation dynamics. Appendix~\ref{app:adjoint} characterises a second mechanism, in which the adjoint of the discrete curl, a fixed operator carrying no parameters, rescales the gradient by the wavenumber before it reaches the backbone. The multiresolution read-out of Eq.~(\ref{eq:multires}) responds to this rescaling, with each level supplying a candidate potential.

\section{Limitations and Future Work}
\label{limitation}

{\bfseries The construction is tied to a regular grid.}
Both guarantees rest on the incidence structure of the uniform staggered grid on which the data are stored. Local mass conservation follows because every interior edge appears in two faces of a cell with opposite orientation, so the six face circulations of Eq.~(\ref{eq:curl}) cancel in pairs. Wall impermeability follows because the four edges of a boundary face carry differences of a single vertex potential, so their circulation telescopes to zero. Carrying the same parameterization to an unstructured mesh, to a cut-cell discretization of the building surfaces, or to a curved manifold requires a discrete de~Rham complex on which both relations continue to hold, together with an assembly operator that identifies the boundary edges of that complex. The first relation is available on simplicial complexes through discrete exterior calculus~\citep{arnold2010finite,desbrun2005discrete}; the second depends on how the solid boundary is represented, and constructing it is open.

{\bfseries The roof-level shear layer is difficult for every surrogate considered.}
Section~\ref{main} reports that all four learned methods are least accurate in the layer where the flow separating from the roofs meets the faster air above, and that this layer is also the least steady one. A one-step map trained on a squared-error objective approaches the conditional mean of the next frame, so the component of the flow that fluctuates most is the component it reproduces least. This concerns the training objective and applies to surrogates of this class independently of the output representation. Whether a probabilistic or ensemble objective recovers that component, and at what cost in pointwise error, is left to future work.